\documentclass{article} %
\usepackage{iclr2023_conference,times}
\usepackage{xcolor}
\usepackage{hyperref}
\usepackage{url}
\usepackage{booktabs}
\usepackage{fancyhdr}
\usepackage{lastpage}
\usepackage{graphicx}
\usepackage{amsmath}
\usepackage{amssymb}
\usepackage{wrapfig}
\usepackage{color, colortbl}
\usepackage{amsthm}
\usepackage[font=small,labelfont=bf]{caption}
\usepackage{subcaption}
\usepackage{array, multirow} %
\usepackage{amssymb}
\usepackage[export]{adjustbox}
\usepackage{makecell}
\usepackage{xspace}
\usepackage{csquotes}
\usepackage{multirow}
\usepackage{mathtools}
\usepackage{algorithm}
\usepackage{algpseudocode}
\usepackage{inconsolata}
\definecolor{darkolivegreen}{rgb}{0.33, 0.42, 0.18}
\def\ours{{LiGO}\xspace}

\title{Learning to Grow Pretrained Models for \\ Efficient Transformer Training}

\author{
\hspace{-2mm} \text{Peihao Wang\textsuperscript{\textnormal{1}}\thanks{Work done during an internship at MIT-IBM Watson AI Lab.}\hspace{3mm} Rameswar Panda\textsuperscript{\textnormal{2}} \hspace{2mm} Lucas Torroba Hennigen\textsuperscript{\textnormal{4}} \hspace{2mm} Philip Greengard\textsuperscript{\textnormal{3}}} \\
\hspace{-0mm}\textbf{Leonid Karlinsky\textsuperscript{\textnormal{2}} \hspace{2mm} Rogerio Feris\textsuperscript{\textnormal{2}} \hspace{2mm} David D. Cox\textsuperscript{\textnormal{2}} \hspace{2mm} Zhangyang Wang\textsuperscript{\textnormal{1}} \hspace{2mm} Yoon Kim\textsuperscript{\textnormal{4}}}  \vspace{2mm} \\
\textsuperscript{1}University of Texas at Austin,
\textsuperscript{2}MIT-IBM Watson AI Lab,
\textsuperscript{3}Columbia University,
\textsuperscript{4}MIT  \\
\small{\texttt{\{peihaowang,atlaswang\}@utexas.edu,}}
\small{\texttt{\{rpanda, leonidka, david.d.cox\}@ibm.com,}} \\
\small{\texttt{rsferis@us.ibm.com,}}
\small{\texttt{pg2118@columbia.edu,}}
\small{\texttt{\{lucastor, yoonkim\}@mit.edu}}
}

\newtheorem{definition}{Definition}

\newtheorem{theorem}{Theorem}
\newtheorem{proposition}[theorem]{Proposition}

\newcommand{\Mat}{\boldsymbol}
\newcommand{\Set}{\mathcal}

\newcommand{\real}{\mathbb{R}}

\DeclareMathOperator{\diag}{diag}

\DeclareMathOperator{\V}{vec}

\DeclareMathOperator*{\argmin}{arg\,min}

\usepackage{mathtools}
\usepackage[overload]{empheq}
\usepackage{bbm}

\usepackage{bbm}

\iclrfinalcopy %
\begin{document}
\setlength{\abovedisplayskip}{4pt}
\setlength{\belowdisplayskip}{4pt}
\setlength{\abovedisplayshortskip}{4pt}
\setlength{\belowdisplayshortskip}{4pt}

\maketitle
\vspace{-5mm}
\begin{abstract}
\vspace{-2mm}
Scaling transformers has led to significant breakthroughs in many domains, leading to a paradigm in which larger versions of existing models are trained and released on a periodic basis. New instances of such models are typically trained completely from scratch, despite the fact that they are often just scaled-up versions of their smaller counterparts. How can we use the implicit knowledge in the parameters of smaller, extant models to enable faster training of newer, larger models? This paper describes an approach for accelerating transformer training by {learning to grow} pretrained transformers, where we learn to linearly map  the parameters of the smaller model to initialize the larger model. For tractable learning, we factorize the linear transformation as a composition of  (linear) width- and  depth-growth operators, and further employ a  Kronecker factorization of these growth operators to encode architectural knowledge. Extensive experiments across both language and vision transformers demonstrate that our learned Linear Growth Operator (LiGO)  can save up to $50\%$ computational cost of training from scratch, while also consistently outperforming strong baselines that also reuse smaller pretrained models to initialize larger models.\footnote{Project page: \url{https://vita-group.github.io/LiGO/}}
\vspace{-4mm}

\end{abstract}

\section{Introduction}
\vspace{-2mm}
The transformer architecture \citep{vaswani2017attention} has emerged as a general purpose architecture for modeling many structured domains \citep{devlin2018bert,brown2020language,Rivese2016239118,doso2021vit,touvron2021training}. Perhaps more so than other architectures, the transformer empirically seems to have inductive biases that make it especially amenable to scaling \citep{rosenfeld2019constructive,kaplan2020scaling}, which  has led to a paradigm in which larger versions of smaller, existing models are trained and released on a periodic basis (e.g., the GPT lineage of models \citep{radford2018improving,radford2019language,brown2020language}). New instances of such models are typically trained completely from scratch, despite the fact that they are often scaled-up versions of their smaller counterparts. Given the  compute required to train even the smaller models, we argue that training each model from scratch is wasteful, and that prior knowledge implicit in the parameters of smaller pretrained models should be leveraged to enable faster training of larger models.

One  approach to this problem is through the lens of \emph{model growth}, wherein a smaller model's pretrained parameters are used to initialize a subset of the larger model's parameters. While earlier works generally froze the parameters initialized from the pretrained model and only trained the new (randomly initialized) parameters \citep{fahlman1989,fahlman1990,gutstein2008}, subsequent work has shown that  copying a subset of the pretrained parameters to initialize the new parameters and then  finetuning the entire network significantly accelerates training and sometimes even leads to better performance \citep{chen2015net2net}. When applied to modern transformers, these mechanisms roughly translate to a  depth-expansion operator in which pretrained models are stacked (or combined with identity layers) to initialize deeper transformers \citep{gong2019efficient,yang2020progressively}, and a width-expansion operator in which the smaller model's matrices are copied to initialize the larger model's matrices (e.g., in block-diagonal fashion) \citep{chen2021bert2bert,gu2020transformer}.

Noting the empirical effectiveness of such recipes, we observe that existing  mechanisms generally do not have a learning component (e.g., randomly copying over neurons for width-expansion or stacking consecutive layers for depth-expansion). This paper instead proposes an efficient, data-driven approach for \emph{learning to grow} transformers. In particular, our  approach frames the problem of initializing the larger model's parameters as learning a linear mapping from the smaller model's parameters, i.e., $\boldsymbol{\Theta}^{(large)} = \Mat{M} \boldsymbol{\Theta}^{(small)}$ where $\boldsymbol{\Theta}^{(small)}$ and $\boldsymbol{\Theta}^{(large)}$ are the vectorized parameters of the small/large models. Due to the high dimensionality of the parameters, this mapping is completely intractable to learn without any restrictions on $\Mat{M}$. We thus factorize the linear mapping to be a composition of sparse width- and depth-expansion operators, $\Mat{M} = \Mat{L}_{depth} \Mat{R}_{width}$, where both width and depth matrices are further factorized to be a Kronecker product of smaller matrices that express architectural knowledge (e.g., through grouping parameters by layers and neurons).
We show that our growth operators can represent existing approaches such as layer-stacking and neuron-copying as special cases. We find that with a small amount of learning on $\Mat{M}$ (e.g., 100 gradient steps) to initialize the larger model, we can significantly accelerate training of both vision and language transformers. Figure~\ref{fig:lego} illustrates our approach.

\begin{figure}
\vspace{-7mm}
    \centering
    \includegraphics[width=0.97\linewidth]{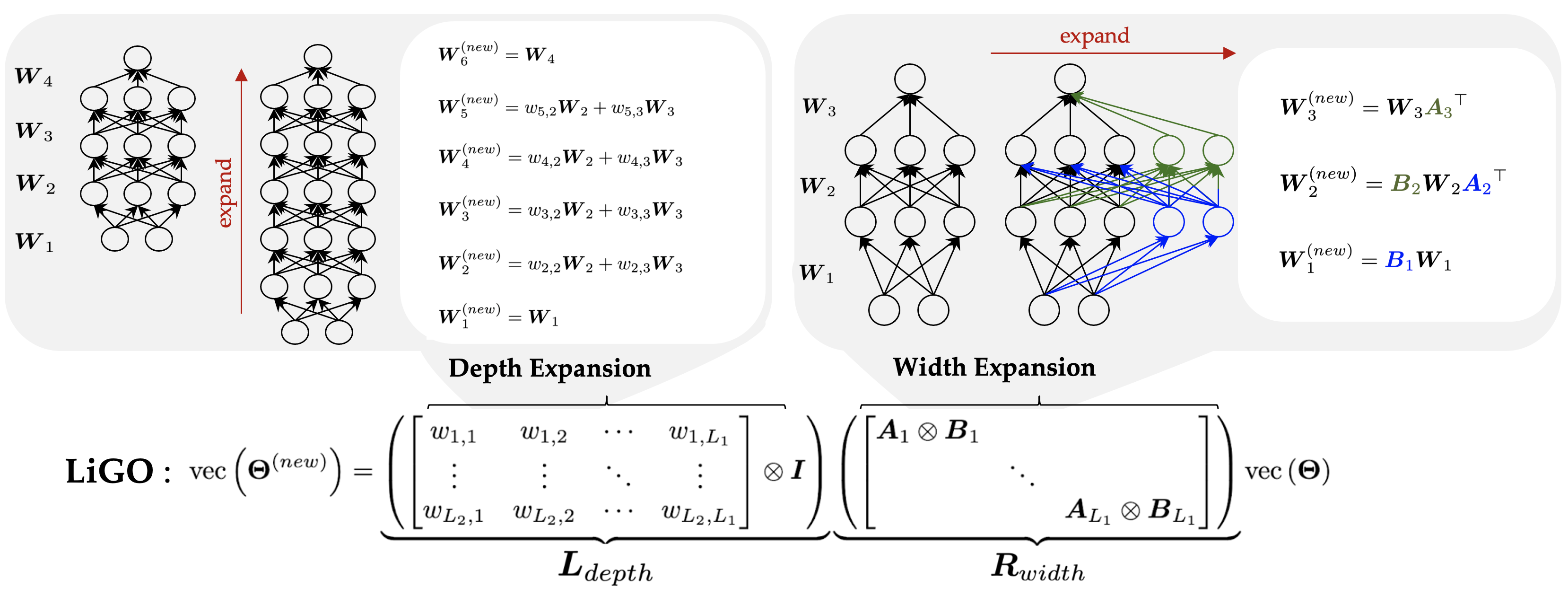}
    \vspace{-3mm}
    \caption{\small Our linear growth operator (LiGO) accelerates training by using the weights of a smaller model $\Mat{\Theta}$ to initialize  the weights of the larger model $\Mat{\Theta}^{(new)}$. LiGO is parameterized as a sparse linear map $\Mat{M}$ that can be decomposed into  width- and depth-expansion operators. The width-operator $\Mat{R}_{width}$ and depth-operator $\Mat{L}_{depth}$  are structured matrices obtained from Kronecker products of smaller matrices which encode architectural knowledge by grouping parameters into layers and neurons. While we show the expansion operators for simple multi-layer perceptrons for illustrative purposes, in practice we apply LiGO to enable faster training of transformer networks. In our approach, we learn the growth matrix $\Mat{M}$ with a 100 steps of SGD, use this to initialize the larger model, and then continue training as usual. Best viewed in color.}
    \label{fig:lego}
    \vspace{-4mm}
\end{figure}

We apply our learned linear growth operator (LiGO) to popular families of  models---BERT~\citep{devlin2018bert}, RoBERTa~\citep{liu2019roberta}, GPT2~\citep{radford2019language}, and ViT~\citep{doso2021vit,touvron2021training,touvron2021going}---and find that \ours can consistently improve transformer training efficiency over the traditional way of training from scratch across domains and model sizes.
For instance, \ours saves $44.7\%$ and $22.5\%$  FLOPs for training BERT-Base and GPT2-Medium from scratch by reusing pretrained smaller models  that are half as big. Similarly, for vision transformers, when using DeiT-S~\citep{touvron2021training} for initialization, \ours yields $55\%$ savings in FLOPs with no performance drop on ImageNet~\citep{deng2009imagenet}. These FLOPs savings directly translate to similar wall clock savings. We further find that models trained using \ours achieve similar performance  to the trained-from-scratch baselines when transferred to  downstream tasks.

\vspace{-2mm}
\section{Related Work}
\vspace{-2mm}

\textbf{Efficient training.} Efficient training of transformers has been studied from multiple perspectives. Some methods that are orthogonal to our work include mixed precision training~\citep{shoeybi2019megatron}, large batch optimization~\citep{you2019large}, distributed training~\citep{huang2019gpipe}, and dropping layers~\citep{zhang2020accelerating} or tokens~\citep{hou2022token}. Knowledge inheritance~\citep{qin2021knowledge} explores knowledge distillation during pretraining to efficiently learn larger transformers. Progressive training, which first trains a small transformer with few layers and then gradually expands by stacking layers, has also been applied to accelerate transformer training \citep{gong2019efficient,yang2020progressively,li2022automated,shen2022staged}. Net2Net~\cite{chen2015net2net} uses function-preserving transformations to grow width by copying neurons and depth by using identity layers. Recently, bert2BERT~\citep{chen2021bert2bert} extends Net2Net to transformers. In contrast to these approaches, our approach {learns} to (linearly) transform the parameters of a smaller model to initialize a larger model. While there is a line of work on learning to grow neural networks in a data-driven way, these methods are in general difficult to apply to modern-scale transformers since they (for example) involve growing a single neuron at a time or employ expensive optimization/search procedures \citep{wei2016network,cai2018efficient,wu2019splitting,wu2021firefly,evci2022gradmax}.

\vspace{-1mm}
\textbf{Network initialization.} Our work is also related to work on  neural network initialization. Existing works include controlling the norm of the parameters~\citep{mishkin2015all,kilcher2018escaping,dai2019nest,wu2019splitting,glorot2010understanding} or replacing the normalization layers~\citep{brock2021high,zhang2019fixup,huang2020improving}. MetaInit~\citep{dauphin2019metainit} proposes an automatic method that optimizes the norms of weight tensors to minimize the gradient quotient on minibatches of random Gaussian samples. GradInit~\citep{zhu2021gradinit} learns to initialize larger networks by adjusting norm of each layer.  Our work focuses on using smaller pretrained transformers to better initialize larger transformers, which remains an understudied problem. 

\vspace{-1mm}
\textbf{Structured matrices.} Finally, our work is also related to structured matrices which are typically used to replace dense weight matrices for reducing training and inference computation cost. Examples include sparse and low rank matrices~\citep{chiu2021low,han2015deep}, Chebyshev matrices~\citep{tang2019chebnet}, Toeplitz matrices~\citep{sindhwani2015structured}, Kronecker-product matrices~\citep{zhang2015fast}, and butterfly matrices~\citep{dao2019learning}. 
A unified framework to learn a broad family of structured matrices is presented in~\cite{sindhwani2015structured}. \citet{dao2022monarch} propose Monarch matrices, which inherit the expressiveness of butterfly matrices and achieve reasonable accuracy-efficiency tradeoffs in many applications. While our approach is inspired by these works, we propose to grow pretrained models by learning structured sparse linear operators with Kronecker factorization, which to our knowledge has not been explored in the literature.

\vspace{-1mm}
\section{Proposed Approach}
\label{sec:methodology}
\vspace{-1mm}

\paragraph{Notation.} We denote the parameters of a neural network with $L$ layers and $D$ dimensions 
as $\Mat{\Theta}_{L,D} = \begin{bmatrix} \Mat{W}_1 & \cdots & \Mat{W}_L \end{bmatrix}^\top \in \real^{LD \times D}$, where $\Mat{W}_l \in \real^{D \times D}$ denotes the weights for the $l$-th layer.\footnote{For notational brevity we assume that each hidden layer has same number of dimensions $D$, but \ours can be straightforwardly generalized to layers with different dimensions (e.g., FFN layers of transformers).}
With slight abuse of notation, we denote the vectorization of $\Mat{\Theta}_{L,D}$ as $\V(\Mat{\Theta}_{L,D})^\top = \begin{bmatrix} \V(\Mat{W}_1)^\top & \cdots & \V(\Mat{W}_L)^\top \end{bmatrix}$.\footnote{We therefore have $\V(\Mat{\Theta}_{L,D})^\top \in \real^{LD^2}$. Our approach is also agnostic with regard to vectorization order.}
Our goal is to re-use the parameters $\Mat{\Theta} = \Mat{\Theta}_{L_1, D_1}$ from a pretrained smaller model to initialize a large model $\Mat{\Theta}^{(new)} = \Mat{\Theta}_{L_2,D_2}$ through a \textit{model growth operator} $M: \real^{L_1 D_1 \times D_1} \rightarrow \real^{L_2 D_2 \times D_2}$ that maps the weights of the smaller network to the weights of the larger one, i.e.,  $\Mat{\Theta}^{(new)} = M(\Mat{\Theta})$ where $L_1 < L_2$ and $D_1 < D_2$. After model growth, we adopt $\Mat{\Theta}^{(new)}$ as the initialization of the large model and train it using standard recipes.

\vspace{-1mm}
\subsection{Existing Growth Operators} \label{sec:prelim}

Existing works have separately established model growth operators for depth ($L_1 < L_2, D_1 = D_2$) and width ($L_1 = L_2, D_1 < D_2$). We summarize these methods below.

\vspace{-1mm}
\textbf{Depth expansion.} StackBERT \citep{gong2019efficient} proposes to duplicate the smaller model 
to double the depth, based on the observation that upper layers share similar functionality with the lower layers. In contrast, interpolation-based depth expansion methods \citep{chang2017multi, dong2020towards} interleave every layer to form a deeper model, which can be roughly interpreted as simulating a finer-grained solution to the original dynamical system from a neural ODE perspective \citep{chen2018neural}. Letting $L_2 = k L_1$, the two methods' growth operators can be formulated as:
\begin{align} \label{eqn:depth_stack}
\text{(StackBERT)}\ \Mat{W}^{(new)}_{l} = \Mat{W}_{l\ \mathrm{mod}\ L_1},
\quad
\text{(Interpolation)}\ \Mat{W}^{(new)}_{i} = \Mat{W}_{\lfloor l / k \rfloor},
\quad
\forall l \in [L_2].
\end{align}

\vspace{-1mm}
\textbf{Width expansion.} Net2Net~\citep{chen2015net2net} expands the width of neural networks by randomly copying neurons while preserving output values via normalization. This can be seen as growing a matrix associated with a particular layer by duplicating the columns and rows of its weight matrix. Suppose a layer has weight matrix $\Mat{W}_{l} \in \real^{D_1 \times D_1}$.\footnote{We define a single layer as $f_l(\Mat{x}) = \Mat{W}_l \Mat{x} + \Mat{b}_l$, where the row number of $\Mat{W}_{l}$ corresponds to the output dimension, and the column number of  $\Mat{W}_{l}$ corresponds to the input dimension.}
To expand it to a matrix $\Mat{W}^{(new)}_{l} \in \real^{D_2 \times D_2}$  ($D_2 > D_1$), Net2Net copies $\Mat{W}_l$ to its upper-left corner of $\Mat{W}^{(new)}_{l}$, fills the new columns via a random selection matrix $\Mat{S}_l$, and finally duplicates and normalizes rows according to the selection matrix from the previous layer. Formally, the growth operator of Net2Net can be written as:
\begin{align} \label{eqn:width_net2net}
\text{(Net2Net)}\ \Mat{W}^{(new)}_l = \begin{bmatrix} \Mat{I} \\ \Mat{S}_{l-1}^\top \end{bmatrix} \Mat{D}_l^{-1} \Mat{W}_l \begin{bmatrix} \Mat{I} & \Mat{S}_l \end{bmatrix}, \quad \Mat{D}_l = \diag(\Mat{S}_{l-1}\Mat{1}) + \Mat{I}, \quad \forall l \in [L_2]
\end{align}
where $\Mat{S}_l \in \{0, 1\}^{D_1 \times (D_2 - D_1)}$ is a random selection matrix. The diagonal of $\Mat{D}_l$ is a $D_1$-dimensional histogram, whose $i$-th entry indicates number of times $i$-th column of $\Mat{W}_l$ was copied.

\vspace{-1mm}
\subsection{Learning to Grow with a Structured  Linear Growth Operator}
While existing operators have been empirically successful in accelerating transformer-based models such as BERT \citep{gong2019efficient,chen2021bert2bert}, we observe that generally do not have a learning component and perform the depth- and width-expansions separately. In this section we introduce a general framework for  learning to grow with a linear growth operator (\ours), which generalizes  existing operators by combining the width- and depth-growth operators in a data-driven way. 

We can formulate the  problem  of initializing the weights of the larger model $\Mat{\Theta}^{(new)}$ from the smaller model $\Mat{\Theta}$ through the following optimization problem,
\begin{align} \label{eqn:learn_init}
\argmin_{M} \, \mathbb{E}_{\Mat{x} \sim \Set{D}} \,\, \mathcal{L}(\Mat{x}; \Mat{\Theta}^{(new)}), \quad
\text{ subject to } \Mat{\Theta}^{(new)} = M (\Mat{\Theta}),
\end{align}
where $\Set{D}$ is the data distribution and $\mathcal{L}$ is the loss function. It is of course intractable to optimize over the entire operator space, and thus we further simplify the function $M$ to be a linear transformation, which results in the following formulation,
\begin{align} \label{eqn:gen_grow}
\V(\Mat{\Theta}^{(new)}) = \V(M(\Mat{\Theta})) = \Mat{M} \V(\Mat{\Theta}), \quad \Mat{M} \in \real^{L_2D_2^2 \times L_1D_1^2}.
\end{align}

This simplified objective is still completely infeasible to apply to contemporary neural networks where $L_1D_1$ can easily be in the hundreds of millions. We therefore propose an  efficient parameterization of $\Mat{M}$ for  tractable learning.

\vspace{-1mm}
\subsubsection{Decomposition along Depth and Width} \label{sec:dw_factor}
Our first step is to decompose the \ours operator as $\Mat{M} = \Mat{L}_{depth} \Mat{R}_{width}$, where $\Mat{L}_{depth}$ and $\Mat{R}_{width}$ expand the depth and width of model separately. Concretely, we decompose $\Mat{M}$ as
\begin{align} \label{eqn:dw_factor}
\Mat{M} = \underbrace{\begin{bmatrix}
\diag(\Mat{\ell}_{1,1}) & \cdots & \diag(\Mat{\ell}_{1,L_1}) \\
\vdots & \ddots & \vdots \\
\diag(\Mat{\ell}_{L_2,1}) & \cdots & \diag(\Mat{\ell}_{L_2, L_1}) \\
\end{bmatrix}}_{\Mat{L}_{depth}}
\underbrace{\begin{bmatrix}
\Mat{R}_{1} & & \\
& \ddots &  \\
& & \Mat{R}_{L_1} \\
\end{bmatrix}}_{\Mat{R}_{width}}.
\end{align}
where $\Mat{R}_l \in \real^{D_2^2 \times D_1^2}$ and $\Mat{\ell}_{i,j} \in \real^{D_2^2}$. 
In the above, $\Mat{L}_{depth}$ is an array of diagonal matrices and $\Mat{R}_{width}$ is a block-diagonal matrix, i.e., both matrices are highly structured and sparse.
When applying $\Mat{R}_{width}$ to weights $\V(\Mat{\Theta})$, the parameters of each layer will be transformed independently via $\V(\Mat{W}^{(new)}_{l}) = \Mat{R}_l \V(\Mat{W}_{l})$ and lifted to a higher dimension. The $l$-th row block of $\Mat{L}_{depth}$ corresponds to  the growth operator of  $l$-th layer, which amounts to linearly combining all layers of the smaller model via $\V(\Mat{W}^{(new)}_{l})_k = \sum_{l'=1}^{L_1} (\Mat{\ell}_{l, l'})_{k} \V(\Mat{W}_{l})_{k}$.
By this factorization, we can effectively reduce the complexity of the \ours operator from $\mathcal{O}(D_1^2 L_1  D_2^2 L_2)$ to $\mathcal{O}(D_1^2 D_2^2 L_1)$ and encode architectural knowledge by grouping parameters by layers. 
Later in Section~\ref{sec:expressiveness}, this representation is also shown to preserve high representation power owing to its connection with Monarch matrices~\citep{dao2022monarch,dao2019learning}.

\vspace{-1mm}
\subsubsection{Parameter Sharing via Kronecker Factorization}
The above \ours operator  requires $\mathcal{O}(D_1^2 D_2^2 L_1)$ parameters for $\Mat{R}_{width}$ and $\mathcal{O}(L_1 L_2 D_2^2)$ for $\Mat{L}_{depth}$. The width operator $\Mat{R}_{width}$ is thus still prohibitively expensive given that $D_1$ (and $D_2$) can easily be in the hundreds or thousands. In this section, we propose a  Kronecker factorization  to further reduce the number of learnable parameters for each growth operator.

\textbf{Depth.} For depth, we treat an entire layer as a single group and construct a new layer by combining existing layers,  effectively tying parameters for all neurons in same layer. Formally, each block in $\Mat{L}_{depth}$ is simplified to be $\diag(\Mat{\ell}_{i,j}) = w_{i,j} \Mat{I}$. Then the entire matrix can be written as a Kronecker factorization, $\Mat{L}_{depth} = \Mat{w} \otimes \Mat{I}$,
where $\Mat{w} \in \real^{L_2 \times L_1}$ is a matrix whose entry $w_{i,j}$ indicates blending weights of $j$-th layer of the small model to form $i$-th layer of the large model. This strategy reduces the number of parameters in $\Mat{L}_{depth}$ to $\mathcal{O}(L_1 L_2)$, and is shown on left-hand side of Figure~\ref{fig:lego}.

\textbf{Width.} For width, we decompose each diagonal block of width expansion operator $\Mat{R}_{width}$ using the Kronecker factorization $\Mat{R}_l = \Mat{A}_l \otimes \Mat{B}_l$, where $\Mat{A}_l, \Mat{B}_l  \in \real^{D_2 \times D_1}$.
Since $\V(\Mat{C} \Mat{A} \Mat{B}) = (\Mat{B}^\top \otimes \Mat{C}) \V(\Mat{A})$ \citep{schacke2004kronecker}, we then have,
\begin{align}
\label{eqn:kron_R} \Mat{R}_{width} \V(\Mat{\Theta})
&= \begin{bmatrix}
\Mat{A}_1 \otimes \Mat{B}_1 & & \\
& \ddots \\
& & \Mat{\Mat{A}}_{L_1} \otimes \Mat{B}_{L_1} 
\end{bmatrix} \V(\Mat{\Theta}) \\
&= \V\left(\begin{bmatrix}
\Mat{B}_1 \Mat{W}_{1} \Mat{A}_1^\top & \cdots &
\Mat{B}_{L_1} \Mat{W}_{L_1} \Mat{A}_{L_1}^\top
\end{bmatrix}^\top\right).
\end{align}
Here we observe that $\Mat{B}_l \Mat{W}_{l} \Mat{A}_l^\top$ performs in- and out-dimension expansion by $\Mat{A}_l$ and $\Mat{B}_l$, respectively. Each new column/row is a linear combination of columns/rows of small model's weight matrix. This factorization, which can be seen as grouping parameters by \emph{neurons}, reduces the number of parameters to $\mathcal{O}(L_1 D_1 D_2)$. Figure~\ref{fig:lego} (right) illustrates LiGO's width-expansion operator.

Altogether, we obtain the final parameterization of \ours operator $\Mat{M}$:
\begin{align} \label{eqn:genral_form}  
\Mat{M} = \underbrace{\left(\begin{bmatrix}
w_{1,1} & w_{1,2} & \cdots & w_{1,L_1} \\ 
\vdots & \vdots & \ddots & \vdots \\
w_{L_2,1} & w_{L_2,2} & \cdots & w_{L_2,L_1} \\
\end{bmatrix} \otimes \Mat{I} \right)}_{\text{Depth expansion}}
\underbrace{\left(\begin{bmatrix}
\Mat{A}_1 \otimes \Mat{B}_1 & & \\
& \ddots \\
& & \Mat{\Mat{A}}_{L_1} \otimes \Mat{B}_{L_1} 
\end{bmatrix}\right)}_{\text{Width expansion}}
\end{align}
We can exploit the factorization to implement the \ours operator (Eq. \ref{eqn:genral_form}) efficiently. 

\textbf{Training.} \ours expands a model in three steps: (1) for each layer, inserting new rows by linearly combining existing rows through $\Mat{B}_{l}$, (2) for each layer, inserting new columns by linearly combining existing columns through $\Mat{A}_{l}$, and then finally (3) reconstructing each layer by linearly combining the weight matrices with $\Mat{w}$ along the depth.  We then run a few steps (e.g., 100 iterations) of SGD  to optimize $\Mat{M}$, which has negligible compute cost relative to regular training. After obtaining $\Mat{M}$, we initialize large model with $\Mat{M} \V(\Mat{\Theta})$, and train parameters $\Mat{\Theta}^{(new)}$ through SGD as usual. Algorithm~\ref{alg:lego_transformer} summarizes a forward pass of \ours with transformer. Finally, as shown in Appendix~\ref{sec:universality} we note that StackBERT (Eq. \ref{eqn:depth_stack}), Interpolation (Eq. \ref{eqn:depth_stack}), and Net2Net (Eq. \ref{eqn:width_net2net}) are all special cases of \ours (Eq. \ref{eqn:genral_form}) with a particular setting of $\Mat{L}_{depth}$ and $\Mat{R}_{width}$.

\vspace{-3mm}
\subsection{\ours for Transformers}
While \ours can be applied to any multi-layer neural network architecture, in this paper we focus on using \ours to grow transformers which have been shown to be particularly amenable to scaling. Below we briefly describe how \ours is applied to the main transformer embedding/attention layers and defer further details (e.g., growing bias vectors, layer norm parameters) to Appendix~\ref{sec:detailed_trans}.

\textbf{Embedding layer.} The embedding layer can be regarded as a linear layer whose inputs are one-hot vectors. We learn a matrix $\Mat{B}^{(emb)}$ to extend its output dimension. This embedding layer is also  used as the final output layer for our transformer language modeling experiments.

\textbf{Attention and feedforward Layers.} An attention layer consists of multi-head attention weights ($\Mat{W}^Q, \Mat{W}^K, \Mat{W}^V$) and a linear projection ($\Mat{W}^O$).
Let $\Mat{A}^{k}_l$ and $\Mat{B}^{k}_l$ where $k \in \{Q, K, V, O\}$ be the $l$-th layer's in- and out-dimension expansion matrices (Eq. \ref{eqn:kron_R}) for the query, key, value, and projection matrices.
To make sure new input and output channels are aligned across modules, we tie the \ours operator as follows: for all $l \in [L_1]$, (1) $\Mat{A}^{k}_l = (\Mat{B}^{(emb)})^\top$ for $\forall k \in \{Q, K ,V\}$, (2) $\Mat{A}^{O}_l = (\Mat{B}_l^{V})^\top$, (3) $\Mat{B}^{O}_l = \Mat{B}^{(emb)}$. The last constraint is added to take into account the  residual connections \citep{chen2021bert2bert}.
We similarly tie parameters for the feed-forward networks, $\Mat{A}^{(fc1)}_l = (\Mat{B}^{(emb)})^{\top}$, $\Mat{A}^{(fc2)}_l = (\Mat{B}^{(fc1)})^{\top}_l$ and $\Mat{B}^{(fc2)}_l = \Mat{B}^{(emb)}$. 
Since transformers make heavy use of residual layers with skip connections, we found that simply using the same $\Mat{B}^{(emb)}$ to parameterize $\Mat{A}_l^k$ and $\Mat{B}_l^k$ for many layers/modules worked well in practice. This  reduces the number of learnable parameters even further and enables fast learning of $\Mat{M}$ on a small amount of data (100 gradient steps).

\vspace{-1mm}
\subsection{Connection to Monarch Matrices} \label{sec:expressiveness}
As shown in Section~\ref{sec:dw_factor}, our depth-width decomposition factorizes $\Mat{M}$ into a multiplication of two structured sparse matrices.
We examine the expressiveness of this factorized representation by relating it to Monarch matrices \citep{dao2022monarch}, defined below.
\begin{definition}
Let the space of Monarch matrices be $\Set{M} \subseteq \real^{mn_1 \times mn_2}$. Then matrix $\Mat{M} \in \Set{M}$ if $\Mat{M} = \Mat{P}_1 \Mat{L} \Mat{P}_2^\top \Mat{R} = \Mat{P}_1 \diag(\Mat{L}_1, \cdots, \Mat{L}_{n_1}) \Mat{P}_2^\top \diag(\Mat{R}_1, \cdots, \Mat{R}_{n_2})$
where $\Mat{L}_{i} \in \real^{b_1 \times b_2}$, $\Mat{R}_i \in \real^{b_3 \times b_4}$ are dense rectangular matrices, and $n_1 b_2 = n_2 b_3$. $\Mat{P}_1$ is the permutation $\pi(i) = (i - b_1 \lfloor \frac{i}{b_1} \rfloor - 1 )n_1 + \lfloor \frac{i}{b_1} \rfloor + 1$ and $\Mat{P}_2$ is the permutation $\pi(j) = (j - b_2 \lfloor \frac{j}{b_2} \rfloor - 1 )n_1 + \lfloor \frac{j}{b_2} \rfloor + 1$.
\end{definition}
\vspace{-1mm}
It is clear that the block-diagonal matrix $\Mat{R}$ has the identical form to our width growing operator $\Mat{R}_{width}$. By applying the permutation matrices $\Mat{P}_1$ and $\Mat{P}_2$ to $\Mat{L}$, $\Mat{L}$ is transformed into exactly the same form with our depth-growth operator $\Mat{L}_{depth}$ in Eq. \ref{eqn:dw_factor}. This implies that our depth-width decomposition coincides with Monarch sparsification of dense matrices, which generalize  butterfly matrices \citep{dao2019learning} and enjoy rich expressivity properties \citep{dao2020kaleidoscope,dao2022monarch}.

\vspace{-2mm}
\section{Experiments}
\vspace{-2mm}

We conduct experiments to answer three key research questions. \textsc{Q1}: To what extent can \ours  improve the training efficiency (FLOPs and wall time) of transformers compared to training from scratch and other growth operators? \textsc{Q2}: Can \ours be universally effective across transformers from different domains (e.g., language and vision) and sizes? \textsc{Q3}: Can models trained using \ours achieve similar performance compared to the baselines when transferred to downstream tasks?

\vspace{-1mm}
\subsection{Experimental Setup}
\textbf{Datasets.} We follow~\citet{tan2020vokenization} and use the English Wikipedia corpus\footnote{While the original BERT~\citep{devlin2018bert} paper also uses the Toronto Book Corpus~\citep{zhu2015aligning}, we do not include it here since it is no longer publicly available.} for training BERT~\citep{devlin2018bert} and RoBERTa~\citep{liu2019roberta}. We use the public C4~\citep{raffel2020exploring} dataset for training GPT2~\citep{radford2019language}. We use ImageNet~\citep{deng2009imagenet} for training vision transformers. We use GLUE~\citep{wang2018glue}, SQuADv1.1~\citep{rajpurkar2016squad}, and SQuADv2.0~\citep{rajpurkar2018know} for evaluating pretrained BERT models.  We test downstream performance of vision transformers (DeiT~\citep{touvron2021training}) by performing transfer learning on $5$ downstream image classification tasks, including CIFAR10~\citep{krizhevsky2009learning}, CIFAR100~\citep{krizhevsky2009learning}, Flowers102~\citep{nilsback2008automated}, StanfordCars~\citep{krause20133d}, and ChestXRay8~\citep{wang2017chestx}. 

\textbf{Models.} We experiment with growing the following language andd vision transformers: (1) BERT-Small$\rightarrow$BERT-Base, BERT-Base$\rightarrow$BERT-Large, BERT-Small$\rightarrow$BERT-Large; (2) RoBERTa-Small$\rightarrow$RoBERTa-Base for RoBERTa; (3) GPT2-Base$\rightarrow$GPT2-Medium, (4) DeiT-S$\rightarrow$DeiT-B, and (5) CaiT-XS$\rightarrow$CaiT-S. {BERT-Small has $6$ layers with $512$ hidden dimensions, while other named models are their usual sizes.} See Appendix~\ref{sec:models} for full details.

\textbf{Baselines.} We compare our approach with the following baselines: (1) training from scratch baseline where we train the larger transformer without using any smaller pretrained models; (2) progressive training methods designed for growing depth in transformers (StackBERT~\citep{gong2019efficient} and MSLT~\citep{yang2020progressively}); (3) bert2BERT~\citep{chen2021bert2bert} that extends Net2Net~\citep{chen2015net2net} for width expansion and stacking for depth expansion; (4) KI~\citep{qin2021knowledge} which uses distillation for transferring knowledge from the smaller model to the larger model.

\textbf{Implementation details.} 
We always use 100 gradient steps to learn the \ours for all models, which is negligible in terms of FLOPs/wall time compared to full training after initialization. We train both BERT and RoBERTa models for $400$K steps with a warmup of $10$K steps.
We remove the next-sentence prediction task~\citep{liu2019roberta} and use a fixed sequence length of $128$ for pretraining both models. For BERT, we use a batch size of $256$ and a learning rate of $2e^{-4}$, while we use a batch size of $1024$ and a learning rate of $8e^{-4}$ for training RoBERTa models. 

Following~\citet{shen2022staged}, we train GPT2 models with a batch size of $384$ and sequence length of $1024$. For vision transformers, we build our models based on DeiT~\citep{touvron2021training} and CaiT~\citep{touvron2021going}, and apply their default hyper-parameters for training on ImageNet dataset. We train all our vision transformers for $300$ epochs with a batch size of $1024$. For transfer learning with BERT/RoBERTa, we follow \citet{tan2020vokenization} and train for $3$ epochs with a learning rate of $1e^{-4}$ and a batch-size of $32$ for all tasks in GLUE. On SQuAD v1.1 and SQuAD 2.0, we fine-tune for $2$ epochs with a learning rate of $5e^{-5}$ and a batch size of $12$. We run both GLUE and SQuAD evaluations three times with different random seeds and report the mean numbers. For transfer learning experiments on DeiT, we finetune the pretrained models with $1000$ epochs, batch size 768, learning rate $0.01$, and use the same data augmentation in training on ImageNet. We use the same pretraining data and experimental settings for all the baselines (including our approach) for a fair comparison. Note that we include the additional compute required for training \ours in all our tables and figures. However, since our \ours is only trained  for 100 steps, the influence on visualization and quantitative saving percentages is negligible. 

\vspace{-1mm}
\subsection{Results and Analysis}

\begin{figure*}
\vspace{-8mm}
     \centering
     \begin{subfigure}[b]{0.35\textwidth}
         \centering
         \includegraphics[width=\textwidth]{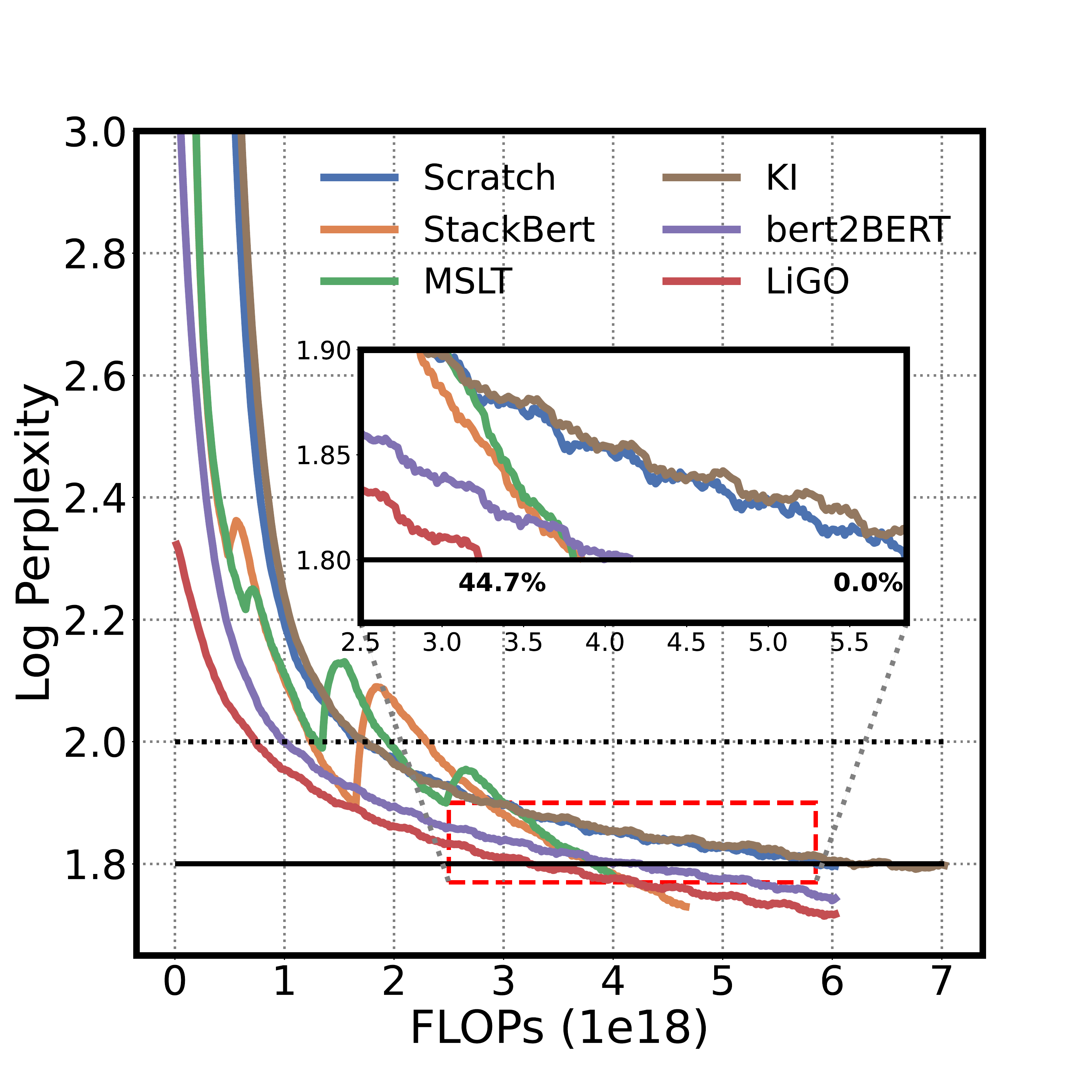} \vspace{-7mm}
         \caption{\scriptsize BERT-Small$\rightarrow$BERT-Base}
     \end{subfigure}
     \hspace{-6mm}
     \begin{subfigure}[b]{0.35\textwidth}
         \centering
         \includegraphics[width=\textwidth]{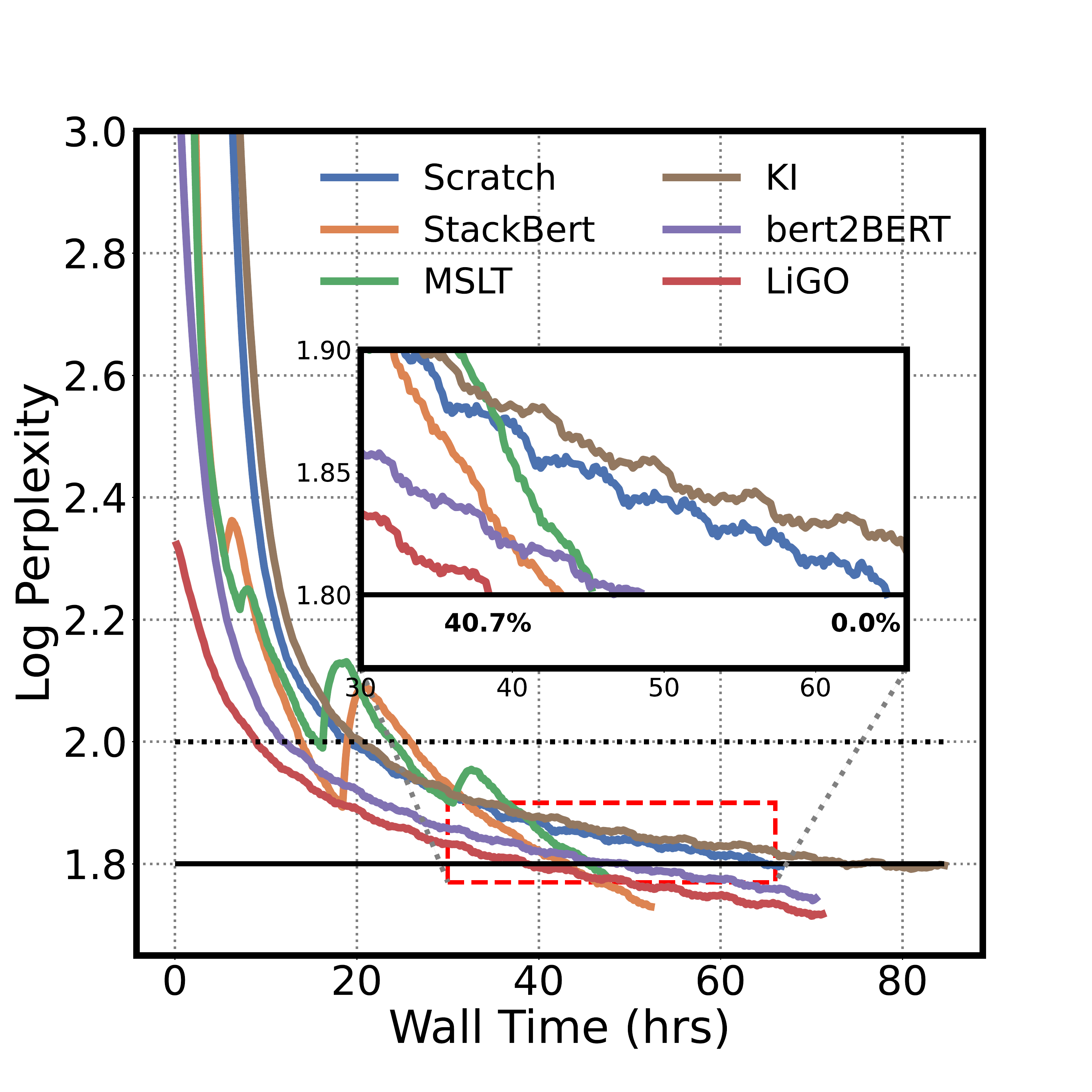} \vspace{-7mm}
         \caption{\scriptsize BERT-Small$\rightarrow$BERT-Base}
     \end{subfigure}
     \hspace{-6mm}
     \begin{subfigure}[b]{0.35\textwidth}
         \centering
         \includegraphics[width=\textwidth]{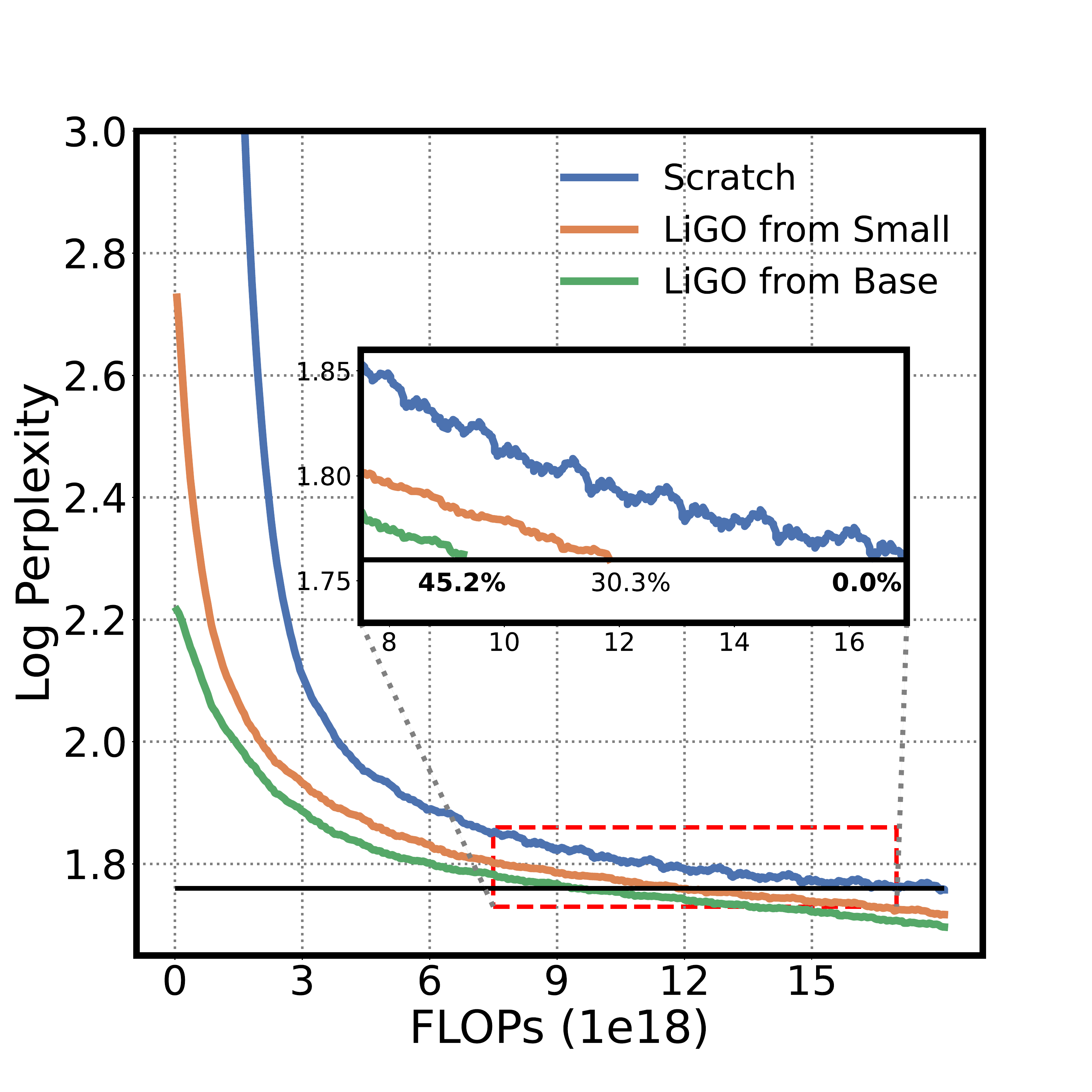} \vspace{-7mm}
         \caption{\scriptsize BERT-\{Small, Base\}$\rightarrow$BERT-Large}
     \end{subfigure} \vspace{-1mm}
        \caption{\small {Results on BERT.} (a-b) shows validation log perplexity vs. FLOPs and wall time respectively for training BERT-Base by reusing BERT-Small. (c) shows log perplexity vs. FLOPs in growing BERT-Small and BERT-Base to BERT-Large. The solid line indicates the final perplexity of the larger model trained from scratch, while the dotted line represents performance of the smaller model trained from scratch. \ours offers about 45\% savings in FLOPs and 40\% savings in wall time over BERT-Base training from scratch. Our approach is also flexible in reusing either BERT-Small or BERT-Base for accelerating BERT-Large training.} 
        \label{fig:bert} 
\end{figure*}

\begin{table}
\definecolor{Gray}{gray}{0.90}
\newcolumntype{a}{>{\columncolor{Gray}}c}
\caption{\small {Downstream transfer learning performance on GLUE and SQuAD.} All of the results are based on BERT-Base models trained using the different baselines. \ours achieves similar or even better performance than the original training from scratch baseline on several downstream tasks, despite improving training efficiency.} \vspace{-4mm}
\label{table:bert}
    \begin{center}
    \resizebox{\linewidth}{!}{
       \begin{tabular}{l|cc|ccccccc|cc|a|a}
             \Xhline{3\arrayrulewidth} 
              \multirow{2}{*}{\textbf{Method}} & \textbf{Savings} & \textbf{Savings} & \textbf{SST-2} & \textbf{MNLI} & \textbf{MRPC} & \textbf{CoLA} & \textbf{QNLI} & \textbf{QQP} & \textbf{STS-B} & \textbf{SQuADv1.1} & \textbf{SQuADv2.0} & \textbf{Avg.} & \textbf{Avg.} \\
    & \textbf{(FLOPs)} & \textbf{(Walltime)} & \textbf{(Acc.)} & \textbf{(Acc.)} & \textbf{(Acc.)} & \textbf{(Acc.)} & \textbf{(Acc.)} & \textbf{(Acc.)} & \textbf{(Acc.)} & \textbf{(F1/EM)} & \textbf{(F1/EM)} & \textbf{GLUE} & \textbf{SQuAD} \\
             \Xhline{2\arrayrulewidth}  
             Scratch & -- &  -- & 88.19 & 78.43 & 85.78 & 62.09 & 87.06 & 87.18 & 86.99 & 86.55 / 77.31 & 71.31 / 67.07 & 82.25 & 78.79 / 72.19           \\
            StackBERT & 34.1\% & 33.3\% & 88.99 & 79.72 & 85.29 & 59.09 & 87.28 & 89.17 & 86.97 & 86.50 / 77.42 & 71.32 / 67.41 & 82.36 & 78.91 / 72.41           \\
            MSLT & 34.9\% & 30.0\% & 88.53 & 78.10 & 82.60 & 64.76 & 83.58 & 88.54 & 85.89 & 86.07 / 76.73 & 70.68 / 67.17    & 81.72 & 78.47 / 71.95  \\
            KI & -5.7\%  & -13.9\% & 88.65 & 78.83 & 83.50 & 64.86 & 86.25 & 88.96 & 87.09 & 84.93 / 76.29 & 71.09 / 67.41 & 82.59 & 78.01 / 71.85 \\
            bert2BERT & 29.0\% & 25.1\% & 88.30 & 80.05 & 85.54 & 61.73 & 88.16 & 86.18 & 87.00 & 86.24 / 77.09 & 71.52 / 66.85 & 82.42 & 78.88 / 71.97 \\
            \Xhline{1\arrayrulewidth}
            \ours & 44.7\% & 40.7\% & 88.42 & 79.29 & 84.31 & 62.09 & 88.07 & 88.81 & 87.00 & 86.28 / 77.45 & 71.24 / 67.17 & 82.57 & 78.76 / 72.31 \\
            \Xhline{3\arrayrulewidth} 
        \end{tabular} 
        }
    \end{center} 
    \vspace{-5mm}
\end{table}

\textbf{BERT.} Figure~\ref{fig:bert} shows the comparison between the different baselines for training BERT models. As seen from Figure~\ref{fig:bert}(a), \ours saves $44.7\%$ computational cost (FLOPs) of training BERT-Base (12 layers, 768 dimensions) from scratch by reusing BERT-Small (6 layers, 512 dimensions). \ours offers $40.7\%$ savings in wall time compared to training from scratch (Figure~\ref{fig:bert}(b)). Among the compared methods, StackBERT is the most competitive in terms of both FLOPs and wall time, although \ours obtains $+10.6\%$ and $+7.2\%$ improvements in FLOPs and wall time on top of StackBERT. Similarly, \ours significantly outperforms the recent  bert2BERT method which saves about $30\%$ computational costs. We observe that KI does not provide any real savings in training as it requires additional computation for knowledge distillation. 
Figure~\ref{fig:bert}(c) shows that our \ours approach is flexible in growing either BERT-Small or BERT-Base for accelerating BERT-Large training. As expected, reusing BERT-Base instead of BERT-Small leads more savings in FLOPs ($45.2\%$ vs $30.3\%$) as BERT-Base contains more implicit knowledge in its parameters.  Table~\ref{table:bert} shows the per-task performance of different BERT-Base models on both GLUE and SQuAD benchmarks, where we find that BERT trained with  \ours achieves very similar performance compared to the baselines on both benchmarks. Finally, in Table~\ref{table:glue_lego_init} of the Appendix~\ref{sec:glue_lego}, we show that growing BERT-Small to BERT-Base with 100 steps of \ours and then finetuning on GLUE tasks \emph{without} additional pretraining outperforms just directly finetuning BERT-Small.

\begin{figure*}
\vspace{-8mm}
     \centering
     \begin{subfigure}[b]{0.35\textwidth}
         \centering
         \includegraphics[width=\textwidth]{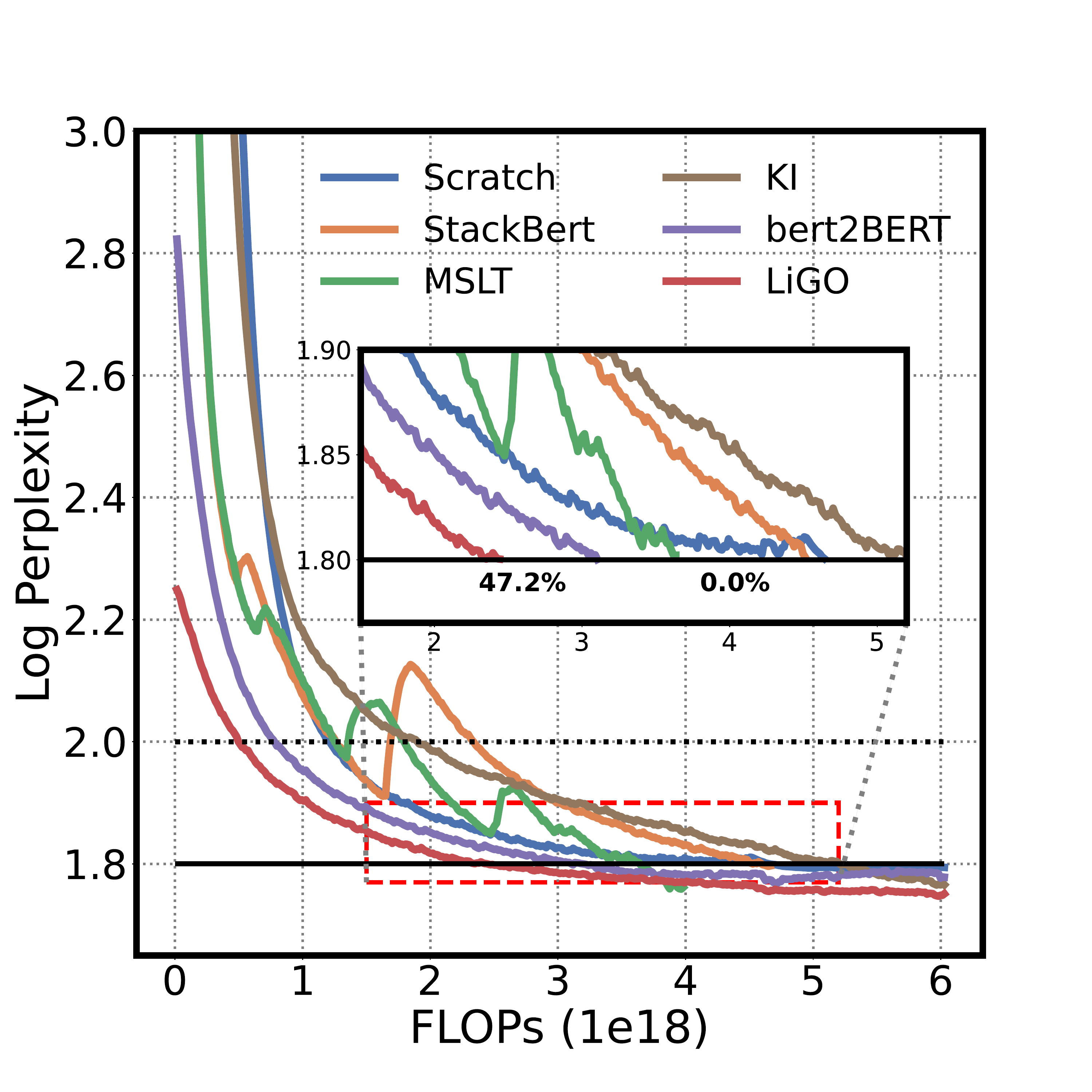} \vspace{-7mm}
         \caption{\scriptsize RoBERTa-Small$\rightarrow$RoBERTa-Base}
     \end{subfigure}
     \hspace{-6mm}
     \begin{subfigure}[b]{0.35\textwidth}
         \centering
         \includegraphics[width=\textwidth]{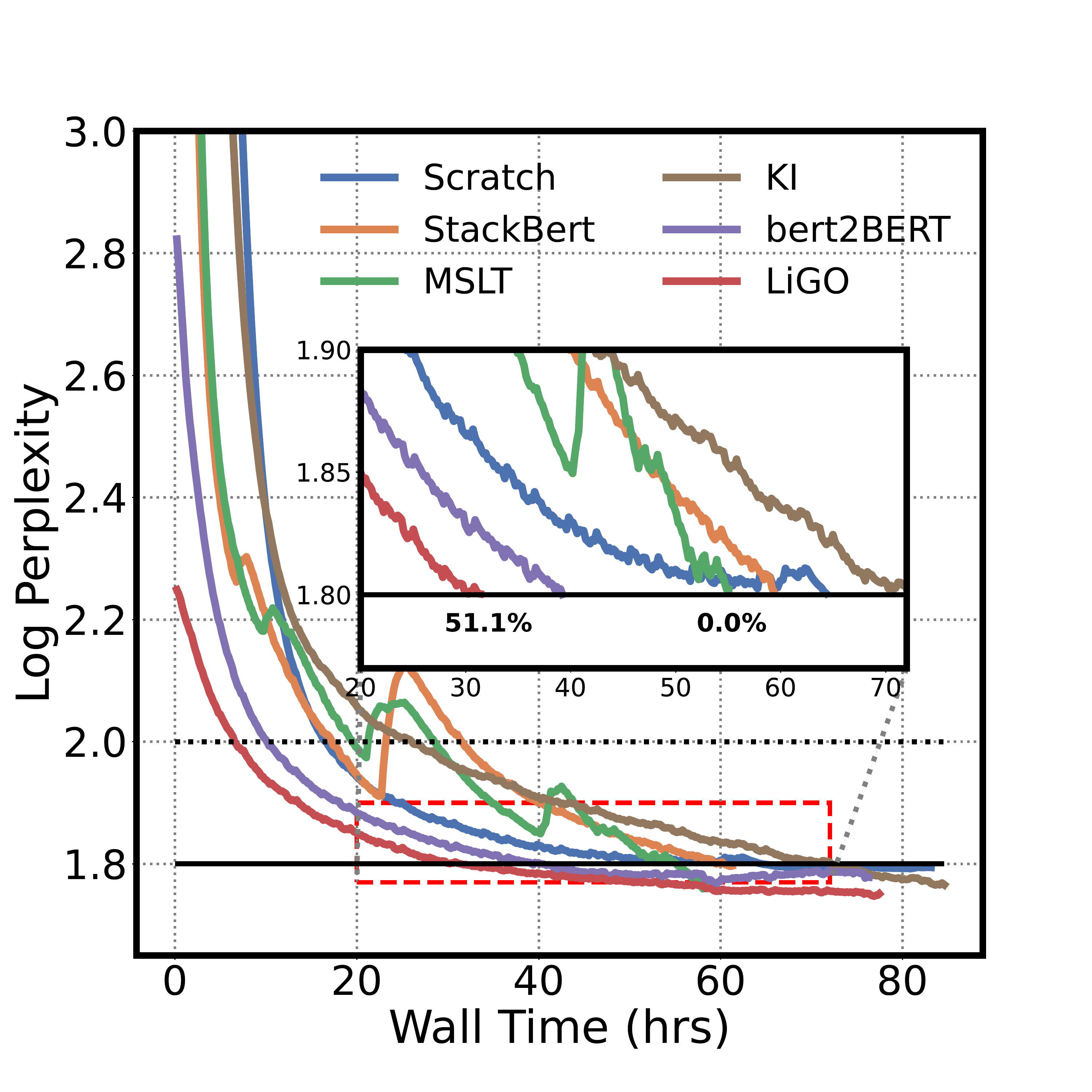} \vspace{-7mm}
         \caption{\scriptsize RoBERTa-Small$\rightarrow$RoBERTa-Base}
     \end{subfigure}
     \hspace{-6mm}
     \begin{subfigure}[b]{0.35\textwidth}
         \centering
         \includegraphics[width=\textwidth]{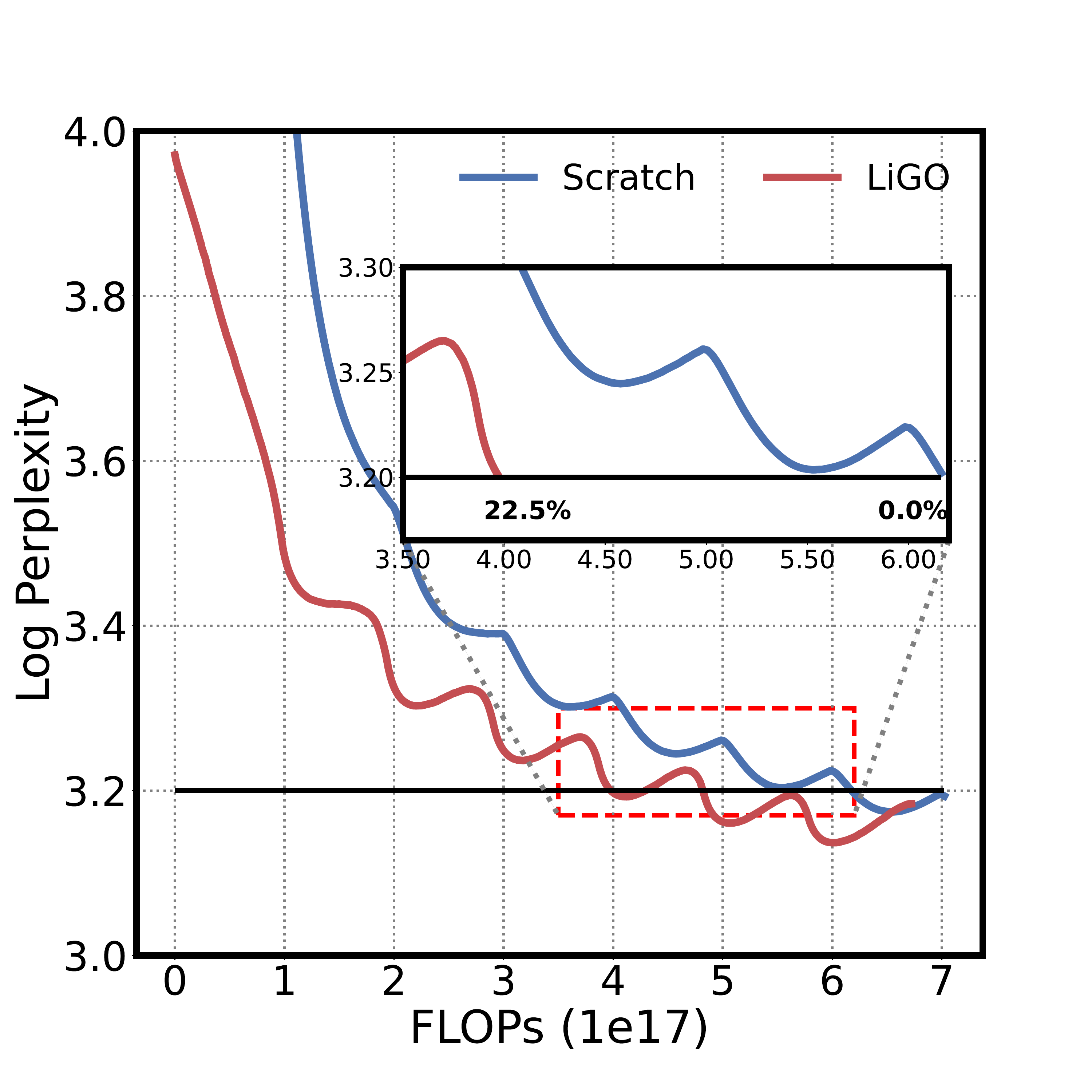} \vspace{-7mm}
         \caption{\scriptsize GPT2-Base$\rightarrow$GPT2-Medium}
     \end{subfigure} \vspace{-1mm}
        \caption{\small {Results on RoBERTa and GPT2.} \ours reduces FLOPs by $47.2\%$ and $22.5\%$ for RoBERTa-Base and GPT2-Medium, , demonstrating its effectiveness across different training strategies and architectures.}
        \label{fig:roberta_gpt} \vspace{-6mm}
\end{figure*}

\textbf{RoBERTa and GPT2.} Figure~\ref{fig:roberta_gpt}(a-b) shows the results on RoBERTa, whose training recipe uses larger batch size and learning rate than BERT. \ours similarly accelerates RoBERTa training, which indicates that our method is robust to optimization hyperparameters.
On GPT2, \ours saves $22.5\%$ computation cost of training GPT2-Medium (345M parameters) by reusing GPT2-Base (117M parameters) (Figure~\ref{fig:roberta_gpt}(c)).
These consistent improvements show that \ours is effective for accelerating transformer training across different model architectures and sizes.  

\begin{wrapfigure}{R}{0.68\textwidth}
     \centering \vspace{-6mm}
     \begin{subfigure}[b]{0.35\textwidth} 
         \centering
         \includegraphics[width=\textwidth]{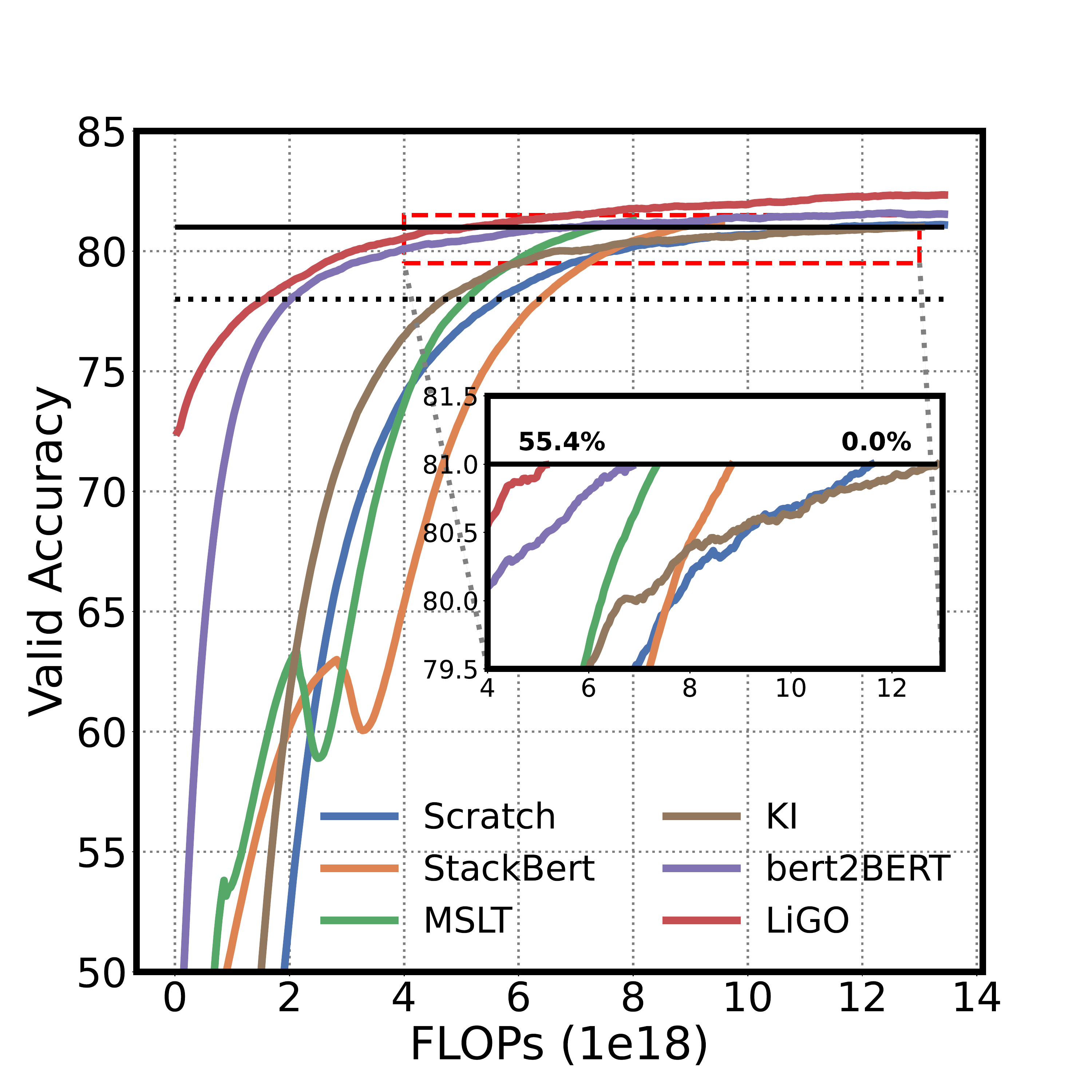} \vspace{-7mm}
         \caption{\scriptsize DeiT-S$\rightarrow$DeiT-B}
     \end{subfigure}
     \hspace{-6mm}
     \begin{subfigure}[b]{0.35\textwidth}
         \centering
         \includegraphics[width=\textwidth]{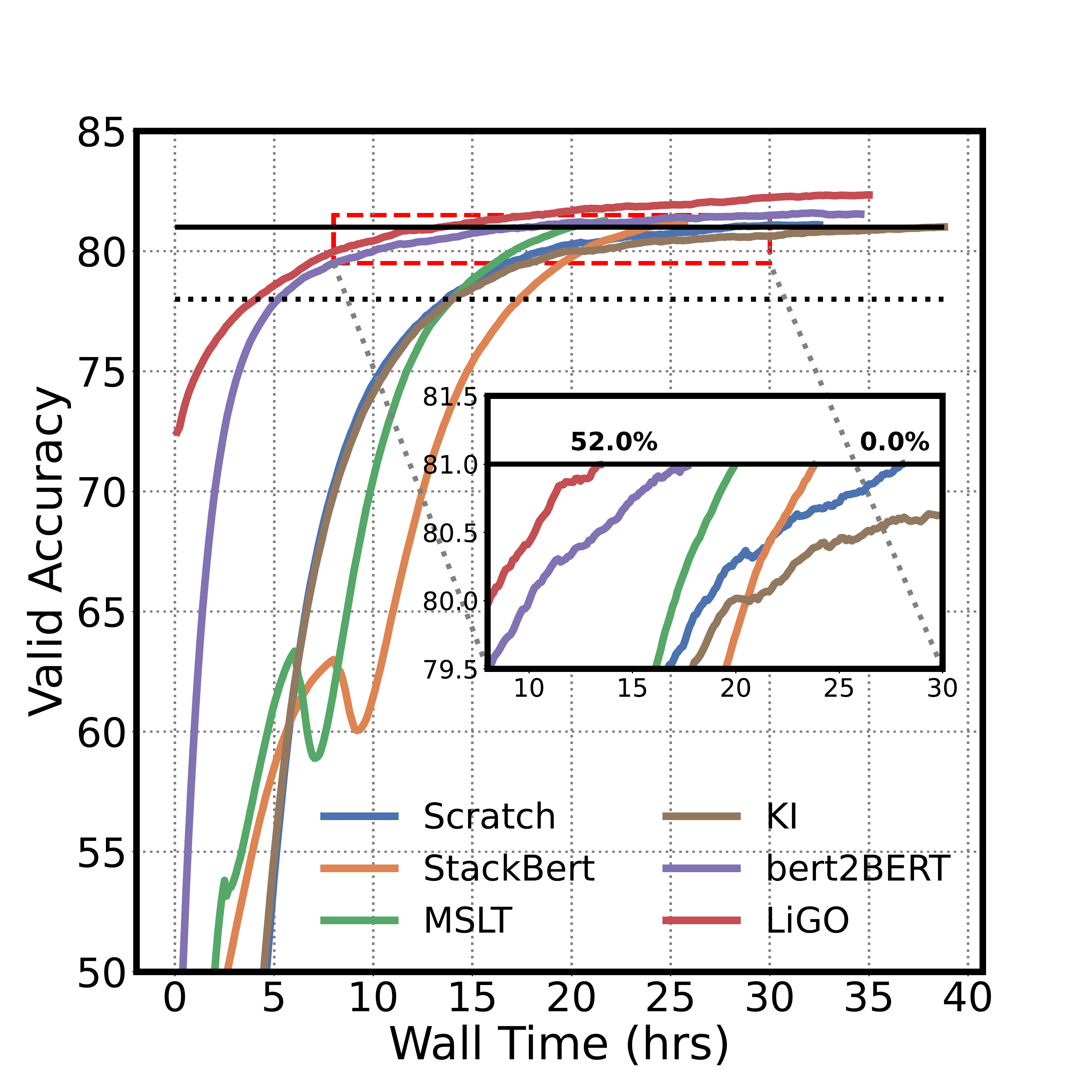} \vspace{-7mm}
         \caption{\scriptsize DeiT-S$\rightarrow$DeiT-B}
     \end{subfigure} \vspace{-2mm}
    \caption{\small Results on DeiT. (a) Accuracy vs. flops and (b) accuracy vs. wall time for training DeiT-B. \ours saves flops and wall time by more than 50\% over training from scratch on ImageNet.}
    \label{fig:deit} \vspace{-2mm}
\end{wrapfigure}

\textbf{Vision Transformers.} 
Figure~\ref{fig:deit} shows that by growing from DeiT-S, \ours can save $55.4$\% FLOPs and $52$\% GPU wall time to reach the same performance of 81\% on ImageNet.
Interestingly, the model initialized by our data-driven growth operator (w/ only 100 gradient steps of tuning) can already achieve $72$\% accuracy at the beginning of training and leads to the final accuracy of $81.7\%$ at the end of the training.
Compared to the next best method, bert2BERT, \ours obtains more than $15\%$ savings, which once again demonstrates the effectiveness of our approach in growing vision transformers as well.
Table~\ref{table:deit} shows that finetuning results on downstream tasks perform on-par  with the model trained from scratch, showing that \ours does not harm the model's generalization capabilities when transferred to downstream datasets.
We also find similar savings in CaiT-XS$\rightarrow$CaiT-S where \ours saves FLOPs by $52.6\%$ and wall time by $46.1\%$ over training CaiT-S from scratch on ImageNet (see Appendix~\ref{sec:cait} for more details).

\begin{wraptable} {r}{0.55\linewidth}
    \begin{center} \vspace{-5mm}
    
            \caption{\small {Transfer learning performance of DeiT-B.} DeiT-B model trained using \ours performs similarly to the original train-from-scratch baseline on all downstream tasks.} \vspace{-2mm}
    \resizebox{\linewidth}{!}{
       \begin{tabular}{l|cc|c|ccccc}
             \Xhline{3\arrayrulewidth} 
              \textbf{Method} & \textbf{FLOPs} & \textbf{Walltime} & \textbf{ImageNet} & \textbf{CIFAR10} & \textbf{CIFAR100} & \textbf{Flowers} & \textbf{Cars} & \textbf{ChestXRay8}  \\
             \Xhline{2\arrayrulewidth}  
             Scratch & -- &  -- & $81.10$ & $99.09$ &  $90.76$ & $97.79$ &  $92.06$ & 55.81 \\
            StackBERT & $23.8\%$ &  $15.1\%$ & $81.21$ & $99.11$ &  $90.80$ & $97.56$ &  $92.09$ & 55.77 \\
            MSLT & $36.7\%$ &  $28.9\%$ & $81.27$ & $99.07$ &  $90.21$ & $97.71$ &  $92.11$ & 55.79\\
            KI & $-11.2\%$ &  $-36.8\%$ & $81.01$ & $98.94$ &  $90.32$ & $97.81$ &  $92.08$ & 55.80\\
            bert2BERT & $40.8\%$ &  $37.0\%$ & $81.59$  & $99.14$ &  $90.69$ & $97.67$ &  $92.15$ & 55.82 \\
            \Xhline{1\arrayrulewidth}
            \ours  & $55.4\%$ &  $52.0\%$ & $81.71$ & $99.12$ &  $90.74$ & $97.77$ &  $92.09$ & 55.82 \\
            \Xhline{3\arrayrulewidth} 
        \end{tabular}
        }
    \label{table:deit}
    \end{center} \vspace{-4mm}
\end{wraptable}

\textbf{Combining with other training strategies.}
We also find that  \ours  can be effectively combined with  orthogonal strategies such as layer dropping~\citep{zhang2020accelerating}, token dropping~\citep{hou2022token}, and staged training~\citep{chen2021bert2bert}. More details are included in Appendix~\ref{sec:ortho}. Figure~\ref{fig:orthogonal} shows that \ours can be combined with other training techniques to improve the computational savings by $4.7\%$, $7.4\%$, and $8.2\%$ with layer dropping, token dropping and staged training, respectively. Following~\citep{chen2021bert2bert}, we also apply staged training strategy to bert2BERT and observe that \ours still outperforms bert2BERT with staged training by $16.7\%$ (see Figure~\ref{fig:orthogonal}(c)).

\begin{figure*}
\vspace{-8mm}
     \centering
     \begin{subfigure}[b]{0.35\textwidth}
         \centering
         \includegraphics[width=\textwidth]{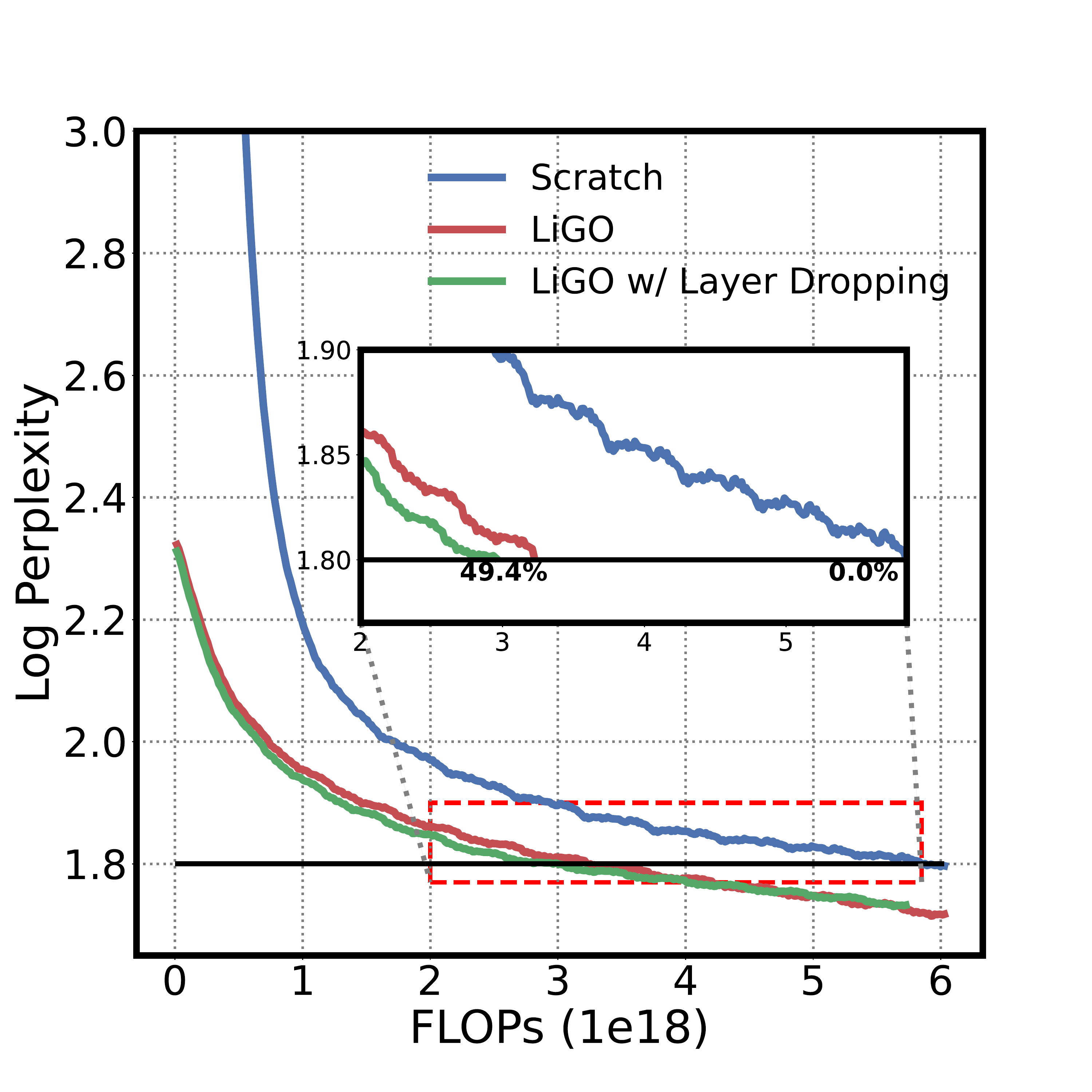} \vspace{-7mm}
         \caption{\scriptsize \ours w/ Layer Dropping}
     \end{subfigure}
     \hspace{-6mm}
     \begin{subfigure}[b]{0.35\textwidth}
         \centering
         \includegraphics[width=\textwidth]{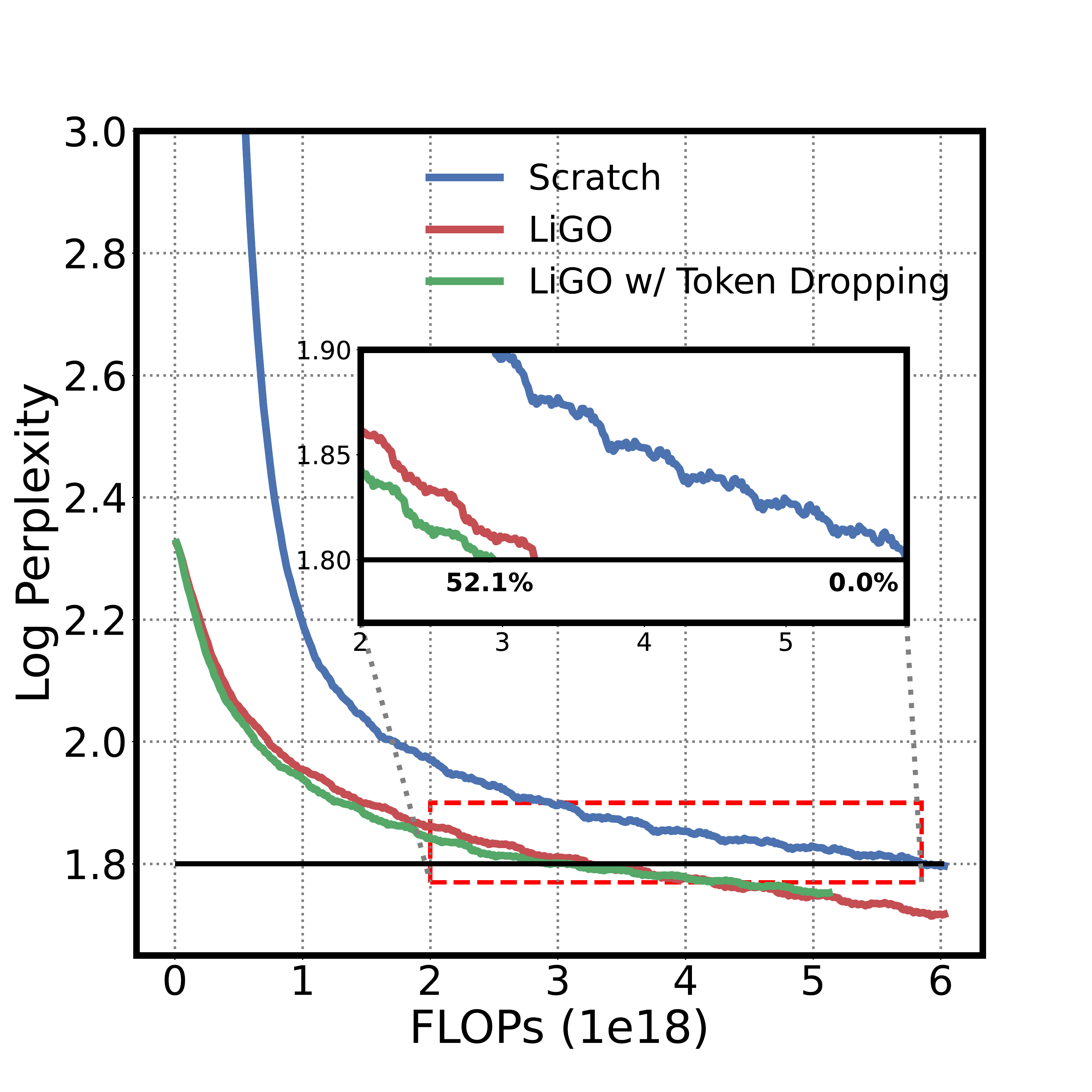} \vspace{-7mm}
         \caption{\scriptsize \ours w/ Token Dropping}
     \end{subfigure}
     \hspace{-6mm}
     \begin{subfigure}[b]{0.35\textwidth}
         \centering
         \includegraphics[width=\textwidth]{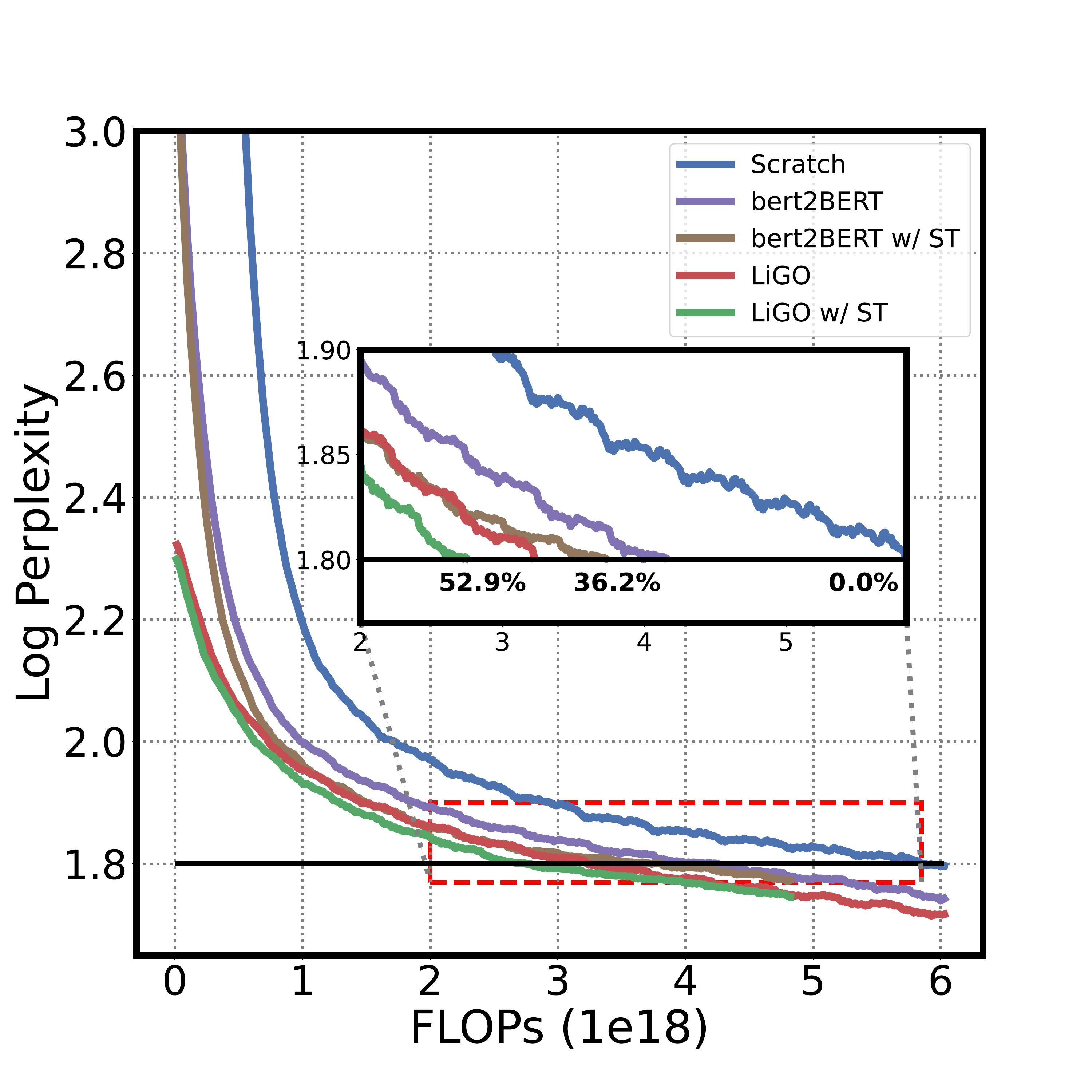} \vspace{-7mm}
         \caption{\scriptsize \ours w/ Staged Training}
     \end{subfigure} \vspace{-1mm}
        \caption{\small {\ours with other efficient training strategies.} Our approach can be  combined with (a) layer dropping, (b) token dropping, and (c) staged training (ST), for further accelerate BERT training.}
        \label{fig:orthogonal} \vspace{-4mm}
\end{figure*}

\vspace{-2mm}
\subsection{Ablation Studies}

\begin{wrapfigure}{R}{0.52\textwidth}
	\centering \vspace{-12mm}
 \begin{subfigure}[b]{0.26\textwidth}
         \centering
         \includegraphics[width=\textwidth]{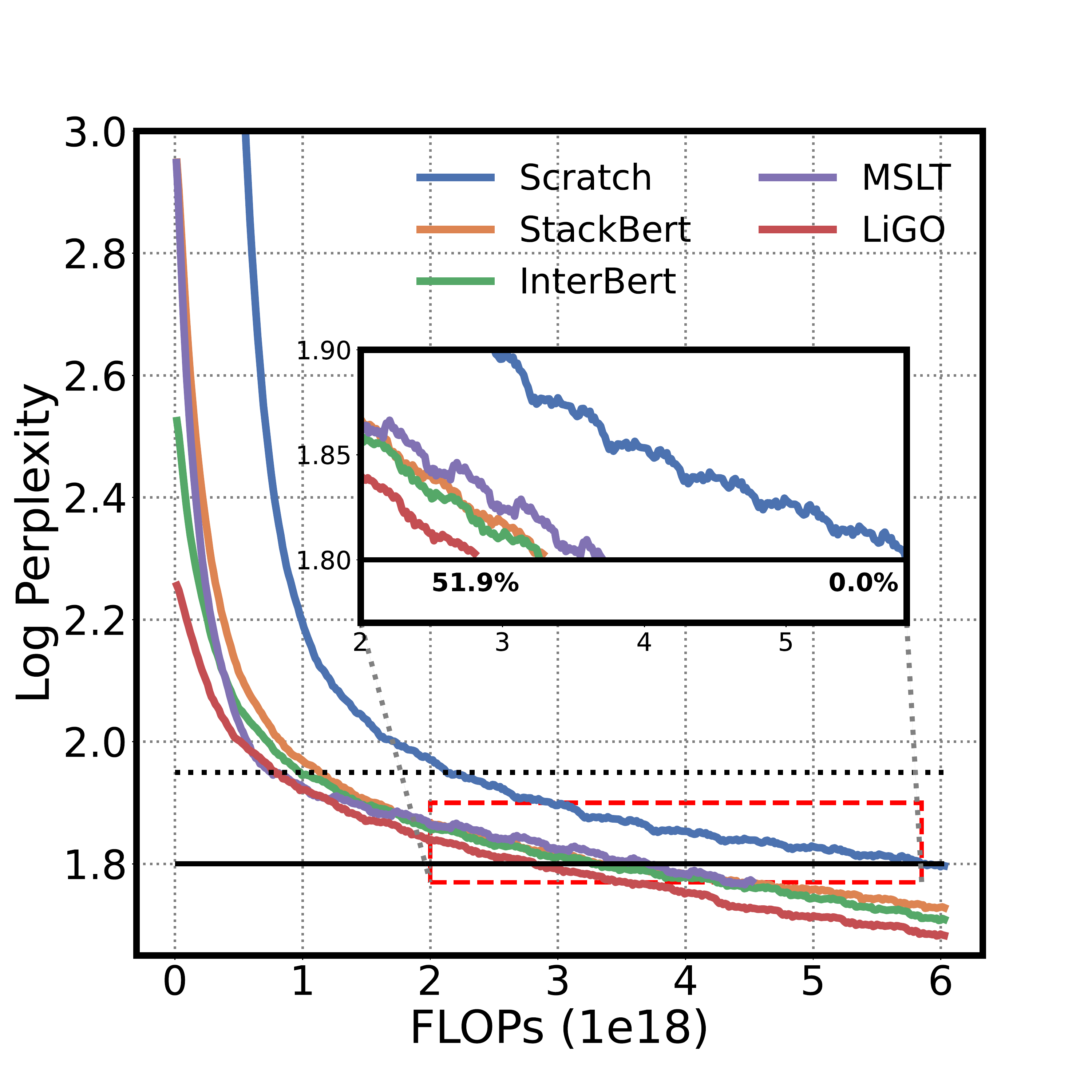} \vspace{-7mm}
         \caption{\tiny BERT(6, 768)$\rightarrow$BERT(12, 768)}
     \end{subfigure}
     \hspace{-4mm}
     \begin{subfigure}[b]{0.26\textwidth}
         \centering
         \includegraphics[width=\textwidth]{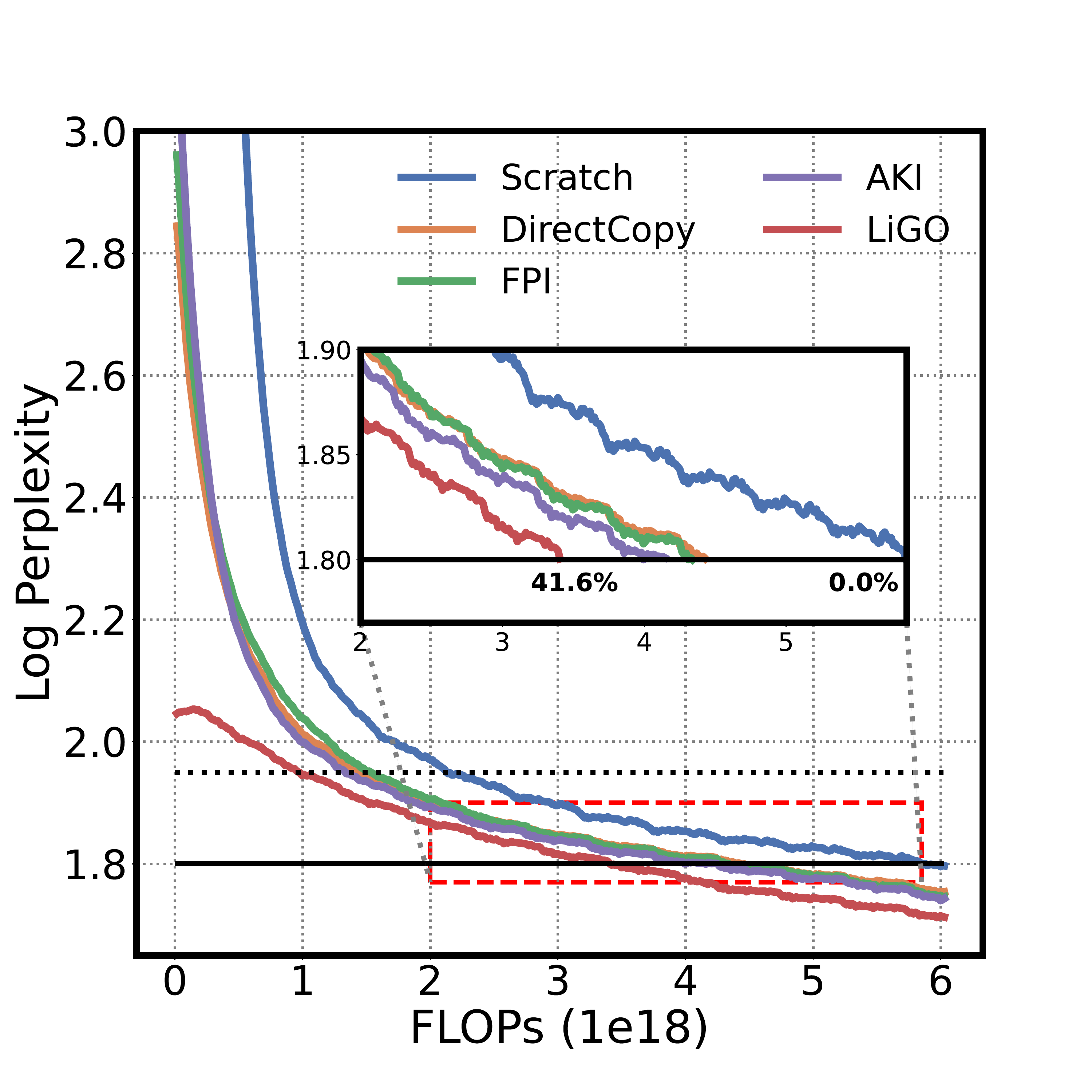} \vspace{-7mm}
         \caption{\tiny BERT(12, 512)$\rightarrow$BERT(12, 768)}
     \end{subfigure} \vspace{-1mm}
	\caption{\small Results on Depth-only and Width-only growth. \ours saves  $51.7\%$ FLOPS when expanding depth-only, and  $41.6\%$ FLOPS when expanding width-only.}
	\label{fig:bert_depth_only} \vspace{-2mm}
\end{wrapfigure}

\textbf{Depth-only expansion.}
We examine the effectiveness of our proposed depth expansion operator ($\Mat{L}_{depth})$ by only growing the depth of BERT from $6$ layers to $12$ layers, i.e, (BERT(6, 768)$\rightarrow$BERT(12, 768).
We compare with stacking~\citep[StackBERT,][]{gong2019efficient}, Interpolation~\citep[InterBERT,][]{chang2017multi,dong2020towards} (see Eq.~\ref{eqn:depth_stack}), and MSLT~\citep{yang2020progressively}.
For \ours, we only apply its $\Mat{L}_{depth}$ component to the pre-trained model weights. Results in Figure~\ref{fig:bert_depth_only}(a) show that a data-driven approach works well even when just growing across the depth dimension.

\begin{wraptable} {r}{0.3\linewidth}
\centering \vspace{-4mm}
    \caption{\small Effect of number of gradient steps. 
    ``+FLOPs" stands for additional flops (in $10^{15}$).}
    \vspace{-2mm}
    \label{table:step}
    \resizebox{\linewidth}{!}{
       \begin{tabular}{l|cc}
             \Xhline{3\arrayrulewidth} 
              \textbf{\# of Steps} & \textbf{+FLOPs} & \textbf{Savings}  \\
             \Xhline{2\arrayrulewidth}  
             100 & 3.61 &  44.7\% \\
             500 & 18.06 &  44.5\% \\
             1000 & 36.13 &  44.2\% \\
             10000 & 361.30 &  38.9\% \\
            \Xhline{3\arrayrulewidth} 
        \end{tabular}
        } \vspace{-3mm}
\end{wraptable}

\textbf{Width-only expansion.}
We also verify the effectiveness of $\Mat{R}_{width}$ by
only extending BERT width from $512$ to $768$, i.e., BERT(12, 512)$\rightarrow$BERT(12, 768).  
We compare \ours based initialization with direct copy \citep{wei2016network}, function preserving initialization  \citep[FPI,][]{chen2015net2net}, and advanced knowledge initialization \citep[AKI,][]{chen2021bert2bert}. LiGO's width expansion component outperforms all other methods, as shown in Figure~\ref{fig:bert_depth_only}(b).

\textbf{Number of growing steps.} Our main experiments just use $100$ gradient steps to grow. We tune our \ours  on the pretraining set for $100$, $500$, $1000$, and $10000$ steps and compute the additional FLOPs for BERT-Small$\rightarrow$BERT-Base training. Table~\ref{table:step} shows that training \ours  within $1000$ steps results in the identical model convergence (reaching $1.8$ PPL at $215$K steps). This suggests tuning model weights under the linear constraints of \ours can achieve faster convergence. Training \ours  for more than $10000$ steps can provide a model with slightly faster convergence ($214$K steps), but results in less saving overall.

\vspace{-1mm}
\section{Conclusion}
\vspace{-1mm}

This paper describes an approach for accelerating transformer training by learning to grow pretrained transformers, where the larger transformer's parameters are initialized as a linear mapping from the smaller pretrained model's parameters, The linear map is factorized  to be a composition of sparse width- and depth-expansion operators with a Kronecker factorization that groups parameters into layers and neurons. We demonstrate the effectiveness of our proposed approach on both language and vision transformers of different sizes, outperforming several competing methods. While our compute resources prevented us from applying \ours to even larger transformers, it would be interesting to see if this can be applied on top of even larger models.

\subsubsection*{Acknowledgments}
PW sincerely thanks Zhen Wang for insightful discussion and providing reference repositories for language model pre-training. PW also appreciates Hao Tan's assistance for reproducing fine-tuning results on GLUE datasets. YK and LTH  were partially supported an MIT-IBM Watson AI grant and an Amazon award. We also acknowledge support from the IBM  Research AI Hardware Center, and the Center for Computational Innovation at Rensselaer Polytechnic Institute for the computational resources on the AiMOS Supercomputer.

\bibliographystyle{iclr2023_conference}
\bibliography{ICLR2023/Main/references}

\begin{thebibliography}{69}
\providecommand{\natexlab}[1]{#1}
\providecommand{\url}[1]{\texttt{#1}}
\expandafter\ifx\csname urlstyle\endcsname\relax
  \providecommand{\doi}[1]{doi: #1}\else
  \providecommand{\doi}{doi: \begingroup \urlstyle{rm}\Url}\fi

\bibitem[Brock et~al.(2021)Brock, De, Smith, and Simonyan]{brock2021high}
Andy Brock, Soham De, Samuel~L Smith, and Karen Simonyan.
\newblock High-performance large-scale image recognition without normalization.
\newblock In \emph{International Conference on Machine Learning}, pp.\
  1059--1071, 2021.

\bibitem[Brown et~al.(2020)Brown, Mann, Ryder, Subbiah, Kaplan, Dhariwal,
  Neelakantan, Shyam, Sastry, Askell, et~al.]{brown2020language}
Tom~B Brown, Benjamin Mann, Nick Ryder, Melanie Subbiah, Jared Kaplan, Prafulla
  Dhariwal, Arvind Neelakantan, Pranav Shyam, Girish Sastry, Amanda Askell,
  et~al.
\newblock Language models are few-shot learners.
\newblock In \emph{Proceedings of NeurIPS}, 2020.

\bibitem[Cai et~al.(2018)Cai, Chen, Zhang, Yu, and Wang]{cai2018efficient}
Han Cai, Tianyao Chen, Weinan Zhang, Yong Yu, and Jun Wang.
\newblock Efficient architecture search by network transformation.
\newblock In \emph{Proceedings of AAAI}, 2018.

\bibitem[Chang et~al.(2017)Chang, Meng, Haber, Tung, and
  Begert]{chang2017multi}
Bo~Chang, Lili Meng, Eldad Haber, Frederick Tung, and David Begert.
\newblock Multi-level residual networks from dynamical systems view.
\newblock \emph{arXiv preprint arXiv:1710.10348}, 2017.

\bibitem[Chen et~al.(2021)Chen, Yin, Shang, Jiang, Qin, Wang, Wang, Chen, Liu,
  and Liu]{chen2021bert2bert}
Cheng Chen, Yichun Yin, Lifeng Shang, Xin Jiang, Yujia Qin, Fengyu Wang, Zhi
  Wang, Xiao Chen, Zhiyuan Liu, and Qun Liu.
\newblock bert2bert: Towards reusable pretrained language models.
\newblock \emph{arXiv preprint arXiv:2110.07143}, 2021.

\bibitem[Chen et~al.(2018)Chen, Rubanova, Bettencourt, and
  Duvenaud]{chen2018neural}
Ricky~TQ Chen, Yulia Rubanova, Jesse Bettencourt, and David~K Duvenaud.
\newblock Neural ordinary differential equations.
\newblock \emph{Advances in neural information processing systems}, 31, 2018.

\bibitem[Chen et~al.(2015)Chen, Goodfellow, and Shlens]{chen2015net2net}
Tianqi Chen, Ian Goodfellow, and Jonathon Shlens.
\newblock Net2net: Accelerating learning via knowledge transfer.
\newblock \emph{arXiv preprint arXiv:1511.05641}, 2015.

\bibitem[Chiu et~al.(2021)Chiu, Deng, and Rush]{chiu2021low}
Justin Chiu, Yuntian Deng, and Alexander Rush.
\newblock Low-rank constraints for fast inference in structured models.
\newblock \emph{Advances in Neural Information Processing Systems},
  34:\penalty0 2887--2898, 2021.

\bibitem[Dai et~al.(2019)Dai, Yin, and Jha]{dai2019nest}
Xiaoliang Dai, Hongxu Yin, and Niraj~K Jha.
\newblock Nest: A neural network synthesis tool based on a grow-and-prune
  paradigm.
\newblock \emph{IEEE Transactions on Computers}, 68\penalty0 (10):\penalty0
  1487--1497, 2019.

\bibitem[Dao et~al.(2019)Dao, Gu, Eichhorn, Rudra, and R{\'e}]{dao2019learning}
Tri Dao, Albert Gu, Matthew Eichhorn, Atri Rudra, and Christopher R{\'e}.
\newblock Learning fast algorithms for linear transforms using butterfly
  factorizations.
\newblock In \emph{International conference on machine learning}, pp.\
  1517--1527, 2019.

\bibitem[Dao et~al.(2020)Dao, Sohoni, Gu, Eichhorn, Blonder, Leszczynski,
  Rudra, and R{\'e}]{dao2020kaleidoscope}
Tri Dao, Nimit~S Sohoni, Albert Gu, Matthew Eichhorn, Amit Blonder, Megan
  Leszczynski, Atri Rudra, and Christopher R{\'e}.
\newblock Kaleidoscope: An efficient, learnable representation for all
  structured linear maps.
\newblock \emph{arXiv preprint arXiv:2012.14966}, 2020.

\bibitem[Dao et~al.(2022)Dao, Chen, Sohoni, Desai, Poli, Grogan, Liu, Rao,
  Rudra, and R{\'e}]{dao2022monarch}
Tri Dao, Beidi Chen, Nimit~S Sohoni, Arjun Desai, Michael Poli, Jessica Grogan,
  Alexander Liu, Aniruddh Rao, Atri Rudra, and Christopher R{\'e}.
\newblock Monarch: Expressive structured matrices for efficient and accurate
  training.
\newblock In \emph{International Conference on Machine Learning}, pp.\
  4690--4721, 2022.

\bibitem[Dauphin \& Schoenholz(2019)Dauphin and
  Schoenholz]{dauphin2019metainit}
Yann~N Dauphin and Samuel Schoenholz.
\newblock Metainit: Initializing learning by learning to initialize.
\newblock \emph{Advances in Neural Information Processing Systems}, 32, 2019.

\bibitem[Deng et~al.(2009)Deng, Dong, Socher, Li, Li, and
  Fei-Fei]{deng2009imagenet}
Jia Deng, Wei Dong, Richard Socher, Li-Jia Li, Kai Li, and Li~Fei-Fei.
\newblock Imagenet: A large-scale hierarchical image database.
\newblock In \emph{2009 IEEE conference on computer vision and pattern
  recognition}, pp.\  248--255, 2009.

\bibitem[Devlin et~al.(2019)Devlin, Chang, Lee, and Toutanova]{devlin2018bert}
Jacob Devlin, Ming-Wei Chang, Kenton Lee, and Kristina Toutanova.
\newblock Bert: Pre-training of deep bidirectional transformers for language
  understanding.
\newblock In \emph{Proceedings of NAACL}, 2019.

\bibitem[Dong et~al.(2020)Dong, Liu, Li, and Shang]{dong2020towards}
Chengyu Dong, Liyuan Liu, Zichao Li, and Jingbo Shang.
\newblock Towards adaptive residual network training: A neural-ode perspective.
\newblock In \emph{International conference on machine learning}, pp.\
  2616--2626. PMLR, 2020.

\bibitem[Dosovitskiy et~al.(2021)Dosovitskiy, Beyer, Kolesnikov, Weissenborn,
  Zhai, Unterthiner, Dehghani, Minderer, Heigold, Gelly, Uszkoreit, and
  Houlsby]{doso2021vit}
Alexey Dosovitskiy, Lucas Beyer, Alexander Kolesnikov, Dirk Weissenborn,
  Xiaohua Zhai, Thomas Unterthiner, Mostafa Dehghani, Matthias Minderer, Georg
  Heigold, Sylvain Gelly, Jakob Uszkoreit, and Neil Houlsby.
\newblock {An Image is Worth 16x16 Words: Transformers for Image Recognition at
  Scale}.
\newblock In \emph{Proceedings of ICLR}, 2021.

\bibitem[Evci et~al.(2022)Evci, Vladymyrov, Unterthiner, van Merri{\"e}nboer,
  and Pedregosa]{evci2022gradmax}
Utku Evci, Max Vladymyrov, Thomas Unterthiner, Bart van Merri{\"e}nboer, and
  Fabian Pedregosa.
\newblock Gradmax: Growing neural networks using gradient information.
\newblock \emph{arXiv preprint arXiv:2201.05125}, 2022.

\bibitem[Fahlman(1990)]{fahlman1990}
Scott Fahlman.
\newblock The recurrent cascade-correlation architecture.
\newblock In \emph{Advances in Neural Information Processing Systems}, 1990.

\bibitem[Fahlman \& Lebiere(1989)Fahlman and Lebiere]{fahlman1989}
Scott Fahlman and Christian Lebiere.
\newblock The cascade-correlation learning architecture.
\newblock In \emph{Advances in Neural Information Processing Systems}, 1989.

\bibitem[Glorot \& Bengio(2010)Glorot and Bengio]{glorot2010understanding}
Xavier Glorot and Yoshua Bengio.
\newblock Understanding the difficulty of training deep feedforward neural
  networks.
\newblock In \emph{Proceedings of the thirteenth international conference on
  artificial intelligence and statistics}, pp.\  249--256, 2010.

\bibitem[Gong et~al.(2019)Gong, He, Li, Qin, Wang, and Liu]{gong2019efficient}
Linyuan Gong, Di~He, Zhuohan Li, Tao Qin, Liwei Wang, and Tieyan Liu.
\newblock Efficient training of bert by progressively stacking.
\newblock In \emph{International conference on machine learning}, pp.\
  2337--2346, 2019.

\bibitem[Gu et~al.(2020)Gu, Liu, Yu, Li, Chen, and Han]{gu2020transformer}
Xiaotao Gu, Liyuan Liu, Hongkun Yu, Jing Li, Chen Chen, and Jiawei Han.
\newblock On the transformer growth for progressive bert training.
\newblock \emph{arXiv preprint arXiv:2010.12562}, 2020.

\bibitem[Gutstein et~al.(2008)Gutstein, Fuentes, , and
  Freudenthal]{gutstein2008}
Steven Gutstein, Olac Fuentes, , and Eric Freudenthal.
\newblock {Knowledge transfer in deep convolutional neural nets}.
\newblock In \emph{Proceedings of International Journal on Artificial
  Intelligence Tools}, 2008.

\bibitem[Han et~al.(2015)Han, Mao, and Dally]{han2015deep}
Song Han, Huizi Mao, and William~J Dally.
\newblock Deep compression: Compressing deep neural networks with pruning,
  trained quantization and huffman coding.
\newblock \emph{arXiv preprint arXiv:1510.00149}, 2015.

\bibitem[Hou et~al.(2022)Hou, Pang, Zhou, Wu, Song, Song, and
  Zhou]{hou2022token}
Le~Hou, Richard~Yuanzhe Pang, Tianyi Zhou, Yuexin Wu, Xinying Song, Xiaodan
  Song, and Denny Zhou.
\newblock Token dropping for efficient bert pretraining.
\newblock \emph{arXiv preprint arXiv:2203.13240}, 2022.

\bibitem[Houlsby et~al.(2019)Houlsby, Giurgiu, Jastrzebski, Morrone,
  De~Laroussilhe, Gesmundo, Attariyan, and Gelly]{houlsby2019parameter}
Neil Houlsby, Andrei Giurgiu, Stanislaw Jastrzebski, Bruna Morrone, Quentin
  De~Laroussilhe, Andrea Gesmundo, Mona Attariyan, and Sylvain Gelly.
\newblock Parameter-efficient transfer learning for nlp.
\newblock In \emph{International Conference on Machine Learning}, 2019.

\bibitem[Huang et~al.(2020)Huang, Perez, Ba, and Volkovs]{huang2020improving}
Xiao~Shi Huang, Felipe Perez, Jimmy Ba, and Maksims Volkovs.
\newblock Improving transformer optimization through better initialization.
\newblock In \emph{International Conference on Machine Learning}, pp.\
  4475--4483, 2020.

\bibitem[Huang et~al.(2019)Huang, Cheng, Bapna, Firat, Chen, Chen, Lee, Ngiam,
  Le, Wu, et~al.]{huang2019gpipe}
Yanping Huang, Youlong Cheng, Ankur Bapna, Orhan Firat, Dehao Chen, Mia Chen,
  HyoukJoong Lee, Jiquan Ngiam, Quoc~V Le, Yonghui Wu, et~al.
\newblock Gpipe: Efficient training of giant neural networks using pipeline
  parallelism.
\newblock \emph{Advances in neural information processing systems}, 32, 2019.

\bibitem[Jia et~al.(2022)Jia, Tang, Chen, Cardie, Belongie, Hariharan, and
  Lim]{jia2022visual}
Menglin Jia, Luming Tang, Bor-Chun Chen, Claire Cardie, Serge Belongie, Bharath
  Hariharan, and Ser-Nam Lim.
\newblock Visual prompt tuning.
\newblock \emph{arXiv preprint arXiv:2203.12119}, 2022.

\bibitem[Kaplan et~al.(2020)Kaplan, McCandlish, Henighan, Brown, Chess, Child,
  Gray, Radford, Wu, and Amodei]{kaplan2020scaling}
Jared Kaplan, Sam McCandlish, Tom Henighan, Tom~B. Brown, Benjamin Chess, Rewon
  Child, Scott Gray, Alec Radford, Jeffrey Wu, and Dario Amodei.
\newblock Scaling laws for neural language models.
\newblock \emph{arXiv preprint arXiv:2001.08361}, 2020.

\bibitem[Kilcher et~al.(2018)Kilcher, B{\'e}cigneul, and
  Hofmann]{kilcher2018escaping}
Yannic Kilcher, Gary B{\'e}cigneul, and Thomas Hofmann.
\newblock Escaping flat areas via function-preserving structural network
  modifications.
\newblock 2018.

\bibitem[Krause et~al.(2013)Krause, Stark, Deng, and Fei-Fei]{krause20133d}
Jonathan Krause, Michael Stark, Jia Deng, and Li~Fei-Fei.
\newblock 3d object representations for fine-grained categorization.
\newblock In \emph{Proceedings of the IEEE international conference on computer
  vision workshops}, pp.\  554--561, 2013.

\bibitem[Krizhevsky et~al.(2009)Krizhevsky, Hinton,
  et~al.]{krizhevsky2009learning}
Alex Krizhevsky, Geoffrey Hinton, et~al.
\newblock Learning multiple layers of features from tiny images.
\newblock 2009.

\bibitem[Lester et~al.(2021)Lester, Al-Rfou, and Constant]{lester2021power}
Brian Lester, Rami Al-Rfou, and Noah Constant.
\newblock The power of scale for parameter-efficient prompt tuning.
\newblock \emph{arXiv preprint arXiv:2104.08691}, 2021.

\bibitem[Li et~al.(2022)Li, Zhuang, Wang, Liang, Chang, and
  Yang]{li2022automated}
Changlin Li, Bohan Zhuang, Guangrun Wang, Xiaodan Liang, Xiaojun Chang, and
  Yi~Yang.
\newblock Automated progressive learning for efficient training of vision
  transformers.
\newblock In \emph{Proceedings of the IEEE/CVF Conference on Computer Vision
  and Pattern Recognition}, pp.\  12486--12496, 2022.

\bibitem[Liu et~al.(2019)Liu, Ott, Goyal, Du, Joshi, Chen, Levy, Lewis,
  Zettlemoyer, and Stoyanov]{liu2019roberta}
Yinhan Liu, Myle Ott, Naman Goyal, Jingfei Du, Mandar Joshi, Danqi Chen, Omer
  Levy, Mike Lewis, Luke Zettlemoyer, and Veselin Stoyanov.
\newblock Roberta: A robustly optimized bert pretraining approach.
\newblock \emph{arXiv preprint arXiv:1907.11692}, 2019.

\bibitem[Mishkin \& Matas(2015)Mishkin and Matas]{mishkin2015all}
Dmytro Mishkin and Jiri Matas.
\newblock All you need is a good init.
\newblock \emph{arXiv preprint arXiv:1511.06422}, 2015.

\bibitem[Nilsback \& Zisserman(2008)Nilsback and
  Zisserman]{nilsback2008automated}
Maria-Elena Nilsback and Andrew Zisserman.
\newblock Automated flower classification over a large number of classes.
\newblock In \emph{2008 Sixth Indian Conference on Computer Vision, Graphics \&
  Image Processing}, pp.\  722--729. IEEE, 2008.

\bibitem[Pfeiffer et~al.(2020)Pfeiffer, Kamath, R{\"u}ckl{\'e}, Cho, and
  Gurevych]{pfeiffer2020adapterfusion}
Jonas Pfeiffer, Aishwarya Kamath, Andreas R{\"u}ckl{\'e}, Kyunghyun Cho, and
  Iryna Gurevych.
\newblock Adapterfusion: Non-destructive task composition for transfer
  learning.
\newblock \emph{arXiv preprint arXiv:2005.00247}, 2020.

\bibitem[Qin et~al.(2021)Qin, Lin, Yi, Zhang, Han, Zhang, Su, Liu, Li, Sun,
  et~al.]{qin2021knowledge}
Yujia Qin, Yankai Lin, Jing Yi, Jiajie Zhang, Xu~Han, Zhengyan Zhang, Yusheng
  Su, Zhiyuan Liu, Peng Li, Maosong Sun, et~al.
\newblock Knowledge inheritance for pre-trained language models.
\newblock \emph{arXiv preprint arXiv:2105.13880}, 2021.

\bibitem[Radford et~al.(2018)Radford, Narasimhan, Salimans, and
  Sutskever]{radford2018improving}
Alec Radford, Karthik Narasimhan, Tim Salimans, and Ilya Sutskever.
\newblock Improving language understanding by generative pre-training.
\newblock 2018.

\bibitem[Radford et~al.(2019)Radford, Wu, Child, Luan, Amodei, and
  Sutskever]{radford2019language}
Alec Radford, Jeff Wu, Rewon Child, David Luan, Dario Amodei, and Ilya
  Sutskever.
\newblock Language models are unsupervised multitask learners.
\newblock 2019.

\bibitem[Raffel et~al.(2020)Raffel, Shazeer, Roberts, Lee, Narang, Matena,
  Zhou, Li, Liu, et~al.]{raffel2020exploring}
Colin Raffel, Noam Shazeer, Adam Roberts, Katherine Lee, Sharan Narang, Michael
  Matena, Yanqi Zhou, Wei Li, Peter~J Liu, et~al.
\newblock Exploring the limits of transfer learning with a unified text-to-text
  transformer.
\newblock \emph{J. Mach. Learn. Res.}, 21\penalty0 (140):\penalty0 1--67, 2020.

\bibitem[Rajpurkar et~al.(2016)Rajpurkar, Zhang, Lopyrev, and
  Liang]{rajpurkar2016squad}
Pranav Rajpurkar, Jian Zhang, Konstantin Lopyrev, and Percy Liang.
\newblock Squad: 100,000+ questions for machine comprehension of text.
\newblock \emph{arXiv preprint arXiv:1606.05250}, 2016.

\bibitem[Rajpurkar et~al.(2018)Rajpurkar, Jia, and Liang]{rajpurkar2018know}
Pranav Rajpurkar, Robin Jia, and Percy Liang.
\newblock Know what you don't know: Unanswerable questions for squad.
\newblock \emph{arXiv preprint arXiv:1806.03822}, 2018.

\bibitem[Rives et~al.(2021)Rives, Meier, Sercu, Goyal, Lin, Liu, Guo, Ott,
  Zitnick, Ma, and Fergus]{Rivese2016239118}
Alexander Rives, Joshua Meier, Tom Sercu, Siddharth Goyal, Zeming Lin, Jason
  Liu, Demi Guo, Myle Ott, C.~Lawrence Zitnick, Jerry Ma, and Rob Fergus.
\newblock Biological structure and function emerge from scaling unsupervised
  learning to 250 million protein sequences.
\newblock \emph{Proceedings of the National Academy of Sciences}, 118\penalty0
  (15), 2021.
\newblock ISSN 0027-8424.
\newblock \doi{10.1073/pnas.2016239118}.

\bibitem[Rosenfeld et~al.(2019)Rosenfeld, Rosenfeld, Belinkov, and
  Shavit]{rosenfeld2019constructive}
Jonathan~S Rosenfeld, Amir Rosenfeld, Yonatan Belinkov, and Nir Shavit.
\newblock A constructive prediction of the generalization error across scales.
\newblock \emph{arXiv preprint arXiv:1909.12673}, 2019.

\bibitem[Schacke(2004)]{schacke2004kronecker}
Kathrin Schacke.
\newblock On the kronecker product.
\newblock \emph{Master's thesis, University of Waterloo}, 2004.

\bibitem[Shen et~al.(2022)Shen, Walsh, Keutzer, Dodge, Peters, and
  Beltagy]{shen2022staged}
Sheng Shen, Pete Walsh, Kurt Keutzer, Jesse Dodge, Matthew Peters, and
  Iz~Beltagy.
\newblock Staged training for transformer language models.
\newblock \emph{arXiv preprint arXiv:2203.06211}, 2022.

\bibitem[Shoeybi et~al.(2019)Shoeybi, Patwary, Puri, LeGresley, Casper, and
  Catanzaro]{shoeybi2019megatron}
Mohammad Shoeybi, Mostofa Patwary, Raul Puri, Patrick LeGresley, Jared Casper,
  and Bryan Catanzaro.
\newblock Megatron-lm: Training multi-billion parameter language models using
  model parallelism.
\newblock \emph{arXiv preprint arXiv:1909.08053}, 2019.

\bibitem[Sindhwani et~al.(2015)Sindhwani, Sainath, and
  Kumar]{sindhwani2015structured}
Vikas Sindhwani, Tara Sainath, and Sanjiv Kumar.
\newblock Structured transforms for small-footprint deep learning.
\newblock \emph{Advances in Neural Information Processing Systems}, 28, 2015.

\bibitem[Tan \& Bansal(2020)Tan and Bansal]{tan2020vokenization}
Hao Tan and Mohit Bansal.
\newblock Vokenization: Improving language understanding with contextualized,
  visual-grounded supervision.
\newblock \emph{arXiv preprint arXiv:2010.06775}, 2020.

\bibitem[Tang et~al.(2019)Tang, Li, and Yu]{tang2019chebnet}
Shanshan Tang, Bo~Li, and Haijun Yu.
\newblock Chebnet: Efficient and stable constructions of deep neural networks
  with rectified power units using chebyshev approximations.
\newblock \emph{arXiv preprint arXiv:1911.05467}, 2019.

\bibitem[Touvron et~al.(2021{\natexlab{a}})Touvron, Cord, Douze, Massa,
  Sablayrolles, and J{\'e}gou]{touvron2021training}
Hugo Touvron, Matthieu Cord, Matthijs Douze, Francisco Massa, Alexandre
  Sablayrolles, and Herv{\'e} J{\'e}gou.
\newblock Training data-efficient image transformers \& distillation through
  attention.
\newblock In \emph{International Conference on Machine Learning}, pp.\
  10347--10357, 2021{\natexlab{a}}.

\bibitem[Touvron et~al.(2021{\natexlab{b}})Touvron, Cord, Sablayrolles,
  Synnaeve, and J{\'e}gou]{touvron2021going}
Hugo Touvron, Matthieu Cord, Alexandre Sablayrolles, Gabriel Synnaeve, and
  Herv{\'e} J{\'e}gou.
\newblock Going deeper with image transformers.
\newblock In \emph{Proceedings of the IEEE/CVF International Conference on
  Computer Vision}, 2021{\natexlab{b}}.

\bibitem[Vaswani et~al.(2017)Vaswani, Shazeer, Parmar, Uszkoreit, Jones, Gomez,
  Kaiser, and Polosukhin]{vaswani2017attention}
Ashish Vaswani, Noam Shazeer, Niki Parmar, Jakob Uszkoreit, Llion Jones,
  Aidan~N Gomez, {\L}ukasz Kaiser, and Illia Polosukhin.
\newblock Attention is {A}ll {Y}ou {N}eed.
\newblock In \emph{Proceedings of NeurIPS}, 2017.

\bibitem[Wang et~al.(2018)Wang, Singh, Michael, Hill, Levy, and
  Bowman]{wang2018glue}
Alex Wang, Amanpreet Singh, Julian Michael, Felix Hill, Omer Levy, and Samuel
  Bowman.
\newblock Glue: A multi-task benchmark and analysis platform for natural
  language understanding.
\newblock In \emph{Proceedings of the 2018 EMNLP Workshop BlackboxNLP:
  Analyzing and Interpreting Neural Networks for NLP}, pp.\  353--355, 2018.

\bibitem[Wang et~al.(2017)Wang, Peng, Lu, Lu, Bagheri, and
  Summers]{wang2017chestx}
Xiaosong Wang, Yifan Peng, Le~Lu, Zhiyong Lu, Mohammadhadi Bagheri, and
  Ronald~M Summers.
\newblock Chestx-ray8: Hospital-scale chest x-ray database and benchmarks on
  weakly-supervised classification and localization of common thorax diseases.
\newblock In \emph{Proceedings of the IEEE conference on computer vision and
  pattern recognition}, pp.\  2097--2106, 2017.

\bibitem[Wei et~al.(2016)Wei, Wang, Rui, and Chen]{wei2016network}
Tao Wei, Changhu Wang, Yong Rui, and Chang~Wen Chen.
\newblock Network morphism.
\newblock In Maria~Florina Balcan and Kilian~Q. Weinberger (eds.),
  \emph{Proceedings of The 33rd International Conference on Machine Learning},
  pp.\  564--572, 2016.

\bibitem[Wu et~al.(2019)Wu, Wang, and Liu]{wu2019splitting}
Lemeng Wu, Dilin Wang, and Qiang Liu.
\newblock Splitting steepest descent for growing neural architectures.
\newblock \emph{Advances in neural information processing systems}, 32, 2019.

\bibitem[Wu et~al.(2021)Wu, Wang, Stone, and Liu]{wu2021firefly}
Lemeng Wu, Dilin Wang, Peter Stone, and Qiang Liu.
\newblock Firefly neural architecture descent: a general approach for growing
  neural networks.
\newblock \emph{Advances in neural information processing systems}, 2021.

\bibitem[Yang et~al.(2020)Yang, Wang, Yang, Li, He, and
  Zhang]{yang2020progressively}
Cheng Yang, Shengnan Wang, Chao Yang, Yuechuan Li, Ru~He, and Jingqiao Zhang.
\newblock Progressively stacking 2.0: A multi-stage layerwise training method
  for bert training speedup.
\newblock \emph{arXiv preprint arXiv:2011.13635}, 2020.

\bibitem[You et~al.(2019)You, Li, Reddi, Hseu, Kumar, Bhojanapalli, Song,
  Demmel, Keutzer, and Hsieh]{you2019large}
Yang You, Jing Li, Sashank Reddi, Jonathan Hseu, Sanjiv Kumar, Srinadh
  Bhojanapalli, Xiaodan Song, James Demmel, Kurt Keutzer, and Cho-Jui Hsieh.
\newblock Large batch optimization for deep learning: Training bert in 76
  minutes.
\newblock \emph{arXiv preprint arXiv:1904.00962}, 2019.

\bibitem[Zhang et~al.(2019)Zhang, Dauphin, and Ma]{zhang2019fixup}
Hongyi Zhang, Yann~N Dauphin, and Tengyu Ma.
\newblock Fixup initialization: Residual learning without normalization.
\newblock \emph{arXiv preprint arXiv:1901.09321}, 2019.

\bibitem[Zhang \& He(2020)Zhang and He]{zhang2020accelerating}
Minjia Zhang and Yuxiong He.
\newblock Accelerating training of transformer-based language models with
  progressive layer dropping.
\newblock \emph{Advances in Neural Information Processing Systems},
  33:\penalty0 14011--14023, 2020.

\bibitem[Zhang et~al.(2015)Zhang, Yu, Guo, Kumar, Wang, and
  Chang]{zhang2015fast}
Xu~Zhang, Felix~X Yu, Ruiqi Guo, Sanjiv Kumar, Shengjin Wang, and Shi-Fu Chang.
\newblock Fast orthogonal projection based on kronecker product.
\newblock In \emph{Proceedings of the IEEE International Conference on Computer
  Vision}, pp.\  2929--2937, 2015.

\bibitem[Zhu et~al.(2021)Zhu, Ni, Xu, Kong, Huang, and
  Goldstein]{zhu2021gradinit}
Chen Zhu, Renkun Ni, Zheng Xu, Kezhi Kong, W~Ronny Huang, and Tom Goldstein.
\newblock Gradinit: Learning to initialize neural networks for stable and
  efficient training.
\newblock \emph{Advances in Neural Information Processing Systems},
  34:\penalty0 16410--16422, 2021.

\bibitem[Zhu et~al.(2015)Zhu, Kiros, Zemel, Salakhutdinov, Urtasun, Torralba,
  and Fidler]{zhu2015aligning}
Yukun Zhu, Ryan Kiros, Rich Zemel, Ruslan Salakhutdinov, Raquel Urtasun,
  Antonio Torralba, and Sanja Fidler.
\newblock Aligning books and movies: Towards story-like visual explanations by
  watching movies and reading books.
\newblock In \emph{Proceedings of the IEEE international conference on computer
  vision}, pp.\  19--27, 2015.

\end{thebibliography}

\appendix

\section{Universality of LiGO Operator} \label{sec:universality}

\begin{proposition} \label{prop:special_case_kron}
StackBERT (Eq. \ref{eqn:depth_stack}), Interpolation (Eq. \ref{eqn:depth_stack}), and Net2Net (Eq. \ref{eqn:width_net2net}) are all the special cases of the LiGO operator (Eq. \ref{eqn:genral_form}).
\end{proposition}

\begin{proof}
We prove Proposition \ref{prop:special_case_kron} by constructing parameters in $\Mat{L}_{depth}$ and $\Mat{R}_{width}$.
\paragraph{Stacking.} Stacking-based methods~\citep{gong2019efficient,yang2020progressively} duplicate the entire lower blocks on top of the small model to the form new layers (Eq. \ref{eqn:depth_stack}). Formally, we show this operation can be done by the following operator:
\begin{align}
\Mat{M} = \underbrace{\begin{bmatrix}
\Mat{I} & & \\
& \Mat{I} & \\
& & \ddots \\
\Mat{I} & & \\
& \Mat{I} & \\
& & \ddots \\
\end{bmatrix}}_{\Mat{L}_{depth}}
\underbrace{\begin{bmatrix}
\Mat{I} & & \\
& \ddots & \\
& & \Mat{I}
\end{bmatrix}}_{\Mat{R}_{width}}
\end{align}

\paragraph{Interpolation.} Interpolation based methods~\citep{chang2017multi,dong2020towards} interleave each layer for twice. We can construct the following matrix to achieve layer interpolation  (Eq. \ref{eqn:depth_stack}).
\begin{align}
\Mat{M} = \underbrace{\begin{bmatrix}
\Mat{I} & & \\
\Mat{I}& & \\
& \Mat{I} & \\
& \Mat{I} & \\
& & \ddots \\
& & \ddots \\
\end{bmatrix}}_{\Mat{L}_{depth}}
\underbrace{\begin{bmatrix}
\Mat{I} & & \\
& \ddots & \\
& & \Mat{I}
\end{bmatrix}}_{\Mat{R}_{width}}
\end{align}

We remark that any rearrangement of layers to construct new layers (mathematically a permutation of existing layers with replacement) can be constructed in a similar way.

\paragraph{Net2Net.} Since we show in Eq. \ref{eqn:kron_R}, the Kronecker factorization on $\Mat{R}_l$ amounts to decomposing the general growth operator into in-dimension and out-dimension expansion. We can construct Net2Net~\citep{chen2015net2net} based growth by simply letting:
\begin{align}
& \Mat{L}_{depth} = \Mat{I} \in \real^{L_1 D_2 \times L_2 D_2}, \quad
\Mat{R}_{width} = \begin{bmatrix}
\Mat{A}_1 \otimes \Mat{B}_1 & & \\
& \ddots & \\
&  & \Mat{A}_{L_1} \otimes \Mat{B}_{L_1}
\end{bmatrix} \\
& \Mat{A}_l = \begin{bmatrix} \Mat{I} \\ \widetilde{\Mat{S}}_{l-1} \end{bmatrix}, \quad
\Mat{B}_l = \begin{bmatrix} \Mat{I} \\ \Mat{S}_{l}, \end{bmatrix}
\end{align}
where $\Mat{S}_l \in \{0, 1\}^{(D_2 - D1) \times D_1}$ is a selection matrix to enlarge the out dimension, and $\widetilde{\Mat{S}}_{l-1} = \Mat{S}_{l-1} \diag(\Mat{1}^\top \Mat{S}_{l-1})^{-1}$ copies the selection from $\Mat{S}_{l-1}$ with normalization to guarantee functionality preserving in expansion.
\end{proof}

\section{Implementation Details}

\subsection{Growing Transformers with LiGO} \label{sec:detailed_trans}

The transformer architecture consists of an embedding layer, multi-block attention layer, and an output layer.
The core ingredient attention block consists of a Multi-Head Attention (MHA) module followed by a FeedForward Network (FFN), with a skip connection across the both blocks.
Applying LiGO requires the following considerations: 

\paragraph{Embedding layer.} For both language and vision transformers, the embedding layer can be regarded as a linear layer, whose inputs are one-hot embeddings in language models. We draw a learnable matrix $\Mat{B}^{(emb)}$ to extend its output dimension.

\paragraph{Multi-head attention blocks.} An attention layer in transformer consists of multi-head attention weights ($\Mat{W}^Q, \Mat{W}^K, \Mat{W}^V$) and a linear projection ($\Mat{W}^O$).
Let $\Mat{A}^{k}_l$ and $\Mat{B}^{k}_l$ with $k \in \{Q, K, V, O\}$ be the in- and out-dimension expansion matrices (Eq. \ref{eqn:kron_R}) for query, key, value, and projection in the $l$-th layer, respectively.
Applying $\Mat{B}^{k}_l$ to $\Mat{W}^k$ ($k \in {Q, K, V}$) constructs new heads by a weighted summation of rows of all existing heads.
To make sure the new input and output channels are aligned across modules, we tie our \ours operator with the following scheme: (1) $\Mat{A}^{k}_l = (\Mat{B}^{(emb)})^\top$ for $\forall k \in \{Q, K ,V\}$, (2) $\Mat{A}^{O}_l = (\Mat{B}^{(V)}_l)^\top$, (3) $\Mat{B}^{O}_l = \Mat{B}^{(emb)}$ for $\forall l \in [L_1]$.
Both the bias  and  layer normalization inherit the associated linear transformations' out-dimension expansion matrices to grow the width.
For depth expansion, each module independently combines the same module from other layers (Eq. \ref{eqn:genral_form}) with learnable coefficients $\Mat{w}$.

\paragraph{Feed-forward networks.}
Each attention block is followed by a two-layer FFN.
Let $\Mat{A}^{k}_l$ and $\Mat{B}^{k}_l$ with $k \in \{fc1, fc2\}$ be the in- and out-dimension expansion matrices (Eq. \ref{eqn:kron_R}) for the first and second FFN layer in the $l$-th layer, respectively.
We tie the parameters for feed-forward networks: $\Mat{A}^{(fc1)}_l = \Mat{B}^{(emb)\top}$, $\Mat{A}^{(fc2)}_l = \Mat{B}^{(fc1)\top}_l$ and $\Mat{B}^{(fc2)}_l = \Mat{B}^{(emb)}$.

\paragraph{Output layer.}
For output head, we have $\Mat{A}^{(out)} = \Mat{B}^{(emb)\top}$, since the output dimension of attention layers are always aligned with  $\Mat{B}^{(emb)}$ by our construction.
The output layer does not need out-dimension expansion. 
Algorithm~\ref{alg:lego_transformer} summarizes \ours for growing transformers.

\begin{algorithm}[t]
\caption{A forward pass of \ours with transformer.}\label{alg:lego_transformer}
\begin{algorithmic}[1]
\State {\textbf{Input}: A small transformer with hidden $D_1$ and number of layer $L_1$. Denote the embedding layer as $\Mat{W}^{(emb)} \in \real^{D_1 \times E}$, attention layers as $\Mat{W}^{Q}_l,\Mat{W}^{K},\Mat{W}^{V}_l,\Mat{W}^{O}_l \in \real^{D_1, \times D_1}$, FFN layers as $\Mat{W}^{(fc1)}_l \in \real^{4D_1 \times D_1}$, $\Mat{W}^{(fc2)}_l \in \real^{D_1 \times 4D_1}$, LayerNorm layers as $\Mat{W}^{(ln1)}_l \in \real^{D_1 \times 4D_1}$, $\Mat{W}^{(ln1)}_l, \Mat{W}^{(ln2)}_l \in \real^{D_1}$, $\forall l \in [L_1]$, the output head $\Mat{W}^{(out)} \in \real^{C
\times D_1}$ }
\State {\textbf{Output}: A large transformer with hidden $D_2$ and number of layer $L_2$. Denote the weight matrices as $\Mat{\Omega}$ with the corresponding superscripts as the small model. }
\State $\Mat{\Omega}^{(emb)} \gets \Mat{B}^{(emb)} \Mat{W}^{(emb)}$
\For{$l = 1, \cdots, L_1$} \Comment{Width Expansion}
\State $\Mat{\Omega}^{Q}_l \gets \Mat{B}^{(Q)} \Mat{W}^{Q}_l \Mat{B}^{(emb)\top}$
\State $\Mat{\Omega}^{K}_l \gets \Mat{B}^{(K)} \Mat{W}^{K}_l \Mat{B}^{(emb)\top}$
\State $\Mat{\Omega}^{V}_l \gets \Mat{B}^{(V)} \Mat{W}^{V}_l \Mat{B}^{(emb)\top}$
\State $\Mat{\Omega}^{O}_l \gets \Mat{B}^{(emb)} \Mat{W}^{V}_l \Mat{B}^{(V)\top}$
\State $\Mat{\Omega}^{(ln1)}_l \gets \Mat{B}^{(emb)} \Mat{W}^{(ln1)}_l$
\State $\Mat{\Omega}^{(fc1)}_l \gets \Mat{B}^{(fc1)} \Mat{W}^{V}_l \Mat{B}^{(emb)\top}$
\State $\Mat{\Omega}^{(fc2)}_l \gets \Mat{B}^{(emb)} \Mat{W}^{V}_l \Mat{B}^{(fc1)\top}$
\State $\Mat{\Omega}^{(ln2)}_l \gets \Mat{B}^{(emb)} \Mat{W}^{(ln2)}_l$
\EndFor
\For{$l = 1, \cdots, L_2$} \Comment{Depth Expansion}
\State $\Mat{\Omega}^{Q}_l \gets \sum_{j=1}^{L_1} w_{l,j}^{Q} \Mat{\Omega}^{Q}_j$
\State $\Mat{\Omega}^{K}_l \gets \sum_{j=1}^{L_1} w_{l,j}^{K} \Mat{\Omega}^{K}_j$
\State $\Mat{\Omega}^{V}_l \gets \sum_{j=1}^{L_1} w_{l,j}^{V} \Mat{\Omega}^{V}_j$
\State $\Mat{\Omega}^{O}_l \gets \sum_{j=1}^{L_1} w_{l,j}^{O} \Mat{\Omega}^{O}_j$
\State $\Mat{\Omega}^{(ln1)}_l \gets \sum_{j=1}^{L_1} w_{l,j}^{(ln1)} \Mat{\Omega}^{(ln1)}_j$
\State $\Mat{\Omega}^{(fc1)}_l \gets \sum_{j=1}^{L_1} w_{l,j}^{(fc1)} \Mat{\Omega}^{(fc1)}_j$
\State $\Mat{\Omega}^{(fc2)}_l \gets \sum_{j=1}^{L_1} w_{l,j}^{(fc2)} \Mat{\Omega}^{(fc2)}_j$
\State $\Mat{\Omega}^{(ln2)}_l \gets \sum_{j=1}^{L_1} w_{l,j}^{(ln2)} \Mat{\Omega}^{(ln2)}_j$
\EndFor
\State $\Mat{\Omega}^{(out)} \gets \Mat{W}^{(out)} \Mat{B}^{(emb)\top}$
\State Train transformer with parameters $\Mat{\Omega}$.
\end{algorithmic}
\end{algorithm}

\subsection{Model Configurations} \label{sec:models}

We summarize the settings of different transformer models used for our experiments in Table~\ref{table:bert_models}.
For BERT and RoBERTa, we re-use the code base provided by \citet{tan2020vokenization}.
For GTP2, we follow the model configuration of OpenAI and use the pre-training code provided by \citet{shen2022staged}.
For DeiT, we use their official codebase \citep{touvron2021training}.

\begin{table}[h!]
\centering
\caption{\small Configuration of different transformers.}
\label{table:bert_models}
\resizebox{1.\linewidth}{!}{
\begin{tabular}{l|cccccc|l|cc}
\Xhline{3\arrayrulewidth} 
& BERT-Small & BERT-Base & RoBERTa-Small & RoBERTa-Base & GPT2-Base & GPT2-Medium & & DeiT-S & DeiT-B \\
\hline
\# layers & 6 &  12 & 6 & 12 & 12 & 24 & \# layers & 12 &  12  \\
\# hidden & 512 & 768  & 512 & 768  & 768 & 1024 & \# hidden & 384 & 768  \\
\# heads & 8 & 12  & 8 & 12  & 12 & 16 & \# heads & 6 & 12 \\
\# vocab & 30522 &  30522 & 50265 & 50265 & 50257 & 50257 & input res. & 224 &  224 \\
seq. length & 128 &  128 & 128 & 128 & 1024 & 1024 & patch size & 16 & 16 \\
\Xhline{3\arrayrulewidth} 
\end{tabular}}
\end{table}

\subsection{Orthogonal Efficient Training Strategies} \label{sec:ortho}

For layer dropping, we follow the same progressive dropping rate schedule with \citet{zhang2020accelerating}, and set the maximum dropping rate to 0.1 to recover the performance. For token dropping, we randomly set 15\% tokens aside in the middle layers. In the first $50$k steps of staged training, only a sub-network is activated and trained, and afterwards, we perform full-model training for 350k steps. 

\begin{figure}
     \centering
     \begin{subfigure}[b]{0.5\textwidth}
         \centering
         \includegraphics[width=\textwidth]{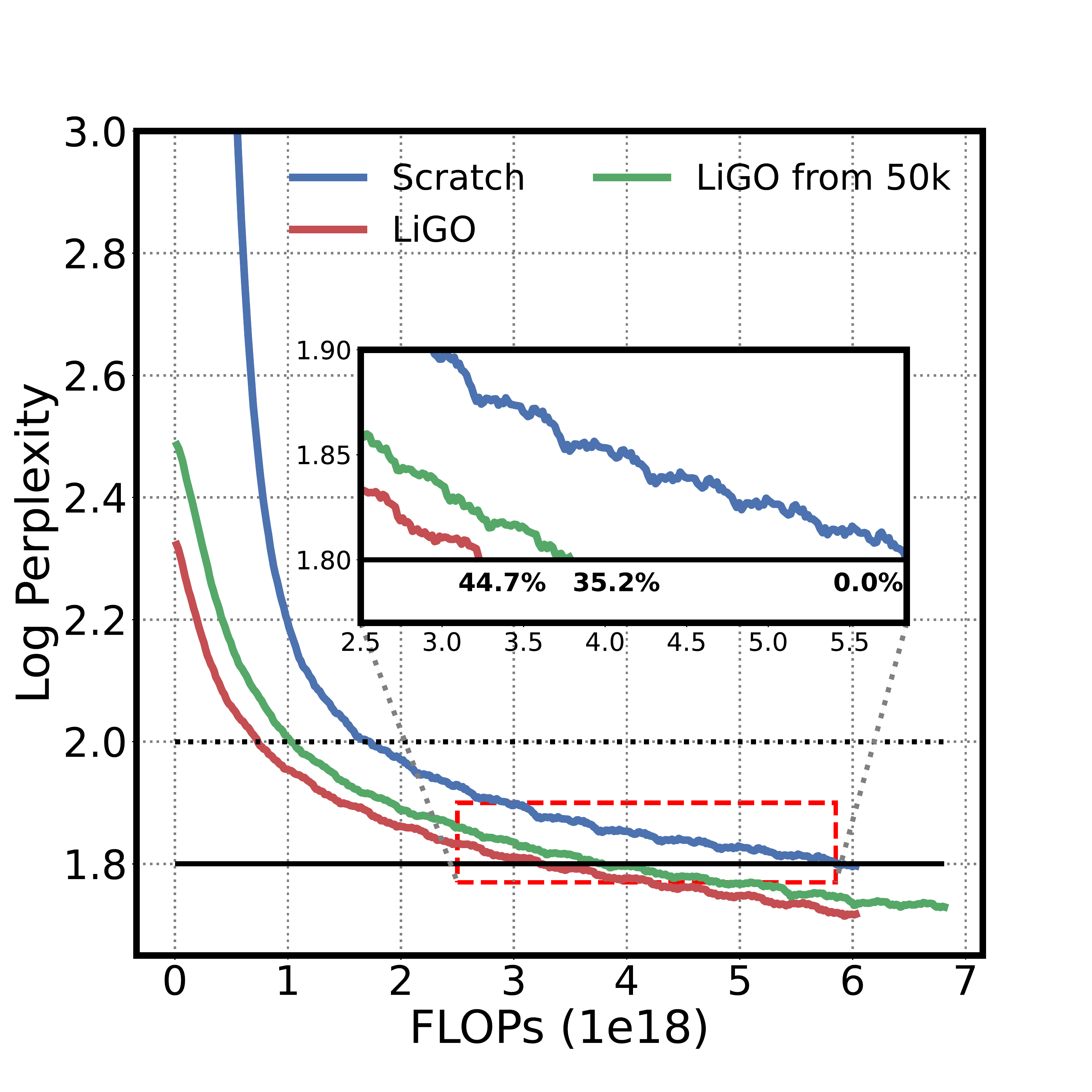} \vspace{-7mm}
         \caption{\scriptsize BERT-Small$\rightarrow$BERT-Base}
     \end{subfigure}
     \hspace{-6mm}
     \begin{subfigure}[b]{0.5\textwidth}
         \centering
         \includegraphics[width=\textwidth]{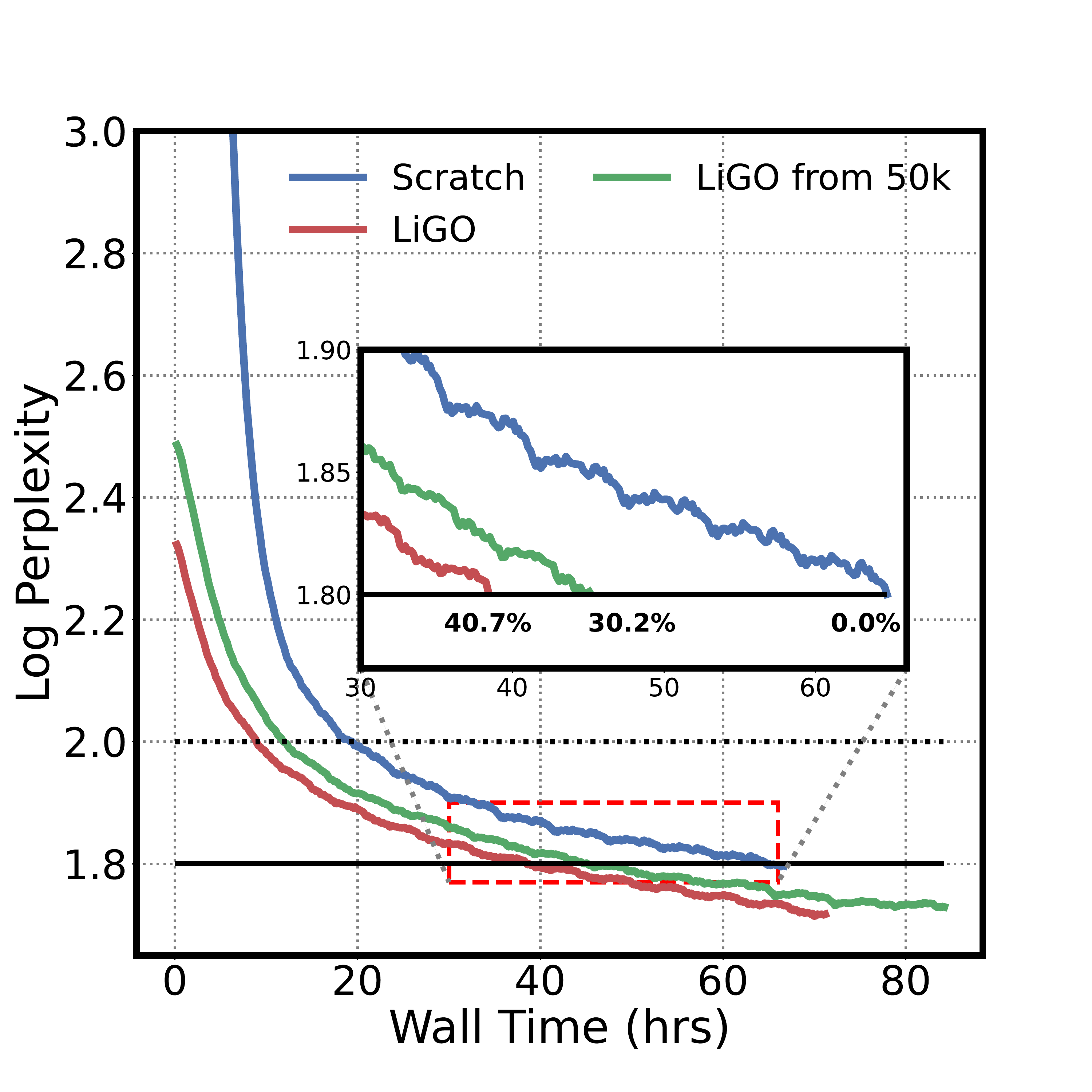} \vspace{-7mm}
         \caption{\scriptsize BERT-Small$\rightarrow$BERT-Base}
     \end{subfigure}
     \hspace{-6mm}
    \caption{\small {Results on BERT-Base by reusing BERT-Small trained for 50k steps.} Instead of training BERT-Base from fully trained BERT-Small, we run LiGO on BERT-Small trained with 50k steps. LiGO offers about 35.2\% savings in FLOPs and 30.2\% savings in wall time over the BERT-Base training from scratch.}
    \label{fig:bert_50k} 
\end{figure}

\section{Additional Experiments}

\subsection{Reusing smaller models trained for only few steps}
\label{sec:small-few}

LiGO focuses on utilizing the knowledge of smaller models that have already been pretrained and available. In this section, we investigate how LiGO can leverage smaller existing models that are only trained for few steps to accelerate training of a larger model. We perform an experiment on BERT-Base by reusing a BERT-Small trained for only 50k steps instead full training for 220k steps as used in our experiments. Figure~\ref{fig:bert_50k} shows that LiGO can still save 35.2\% savings in FLOPs and 30.2\% savings in wall time over the BERT-Base training from scratch.

\subsection{Results on CaiT}
\label{sec:cait}

\begin{figure}
     \centering 
     \vspace{-6mm}
     \begin{subfigure}[b]{0.5\textwidth} 
         \centering
         \includegraphics[width=\textwidth]{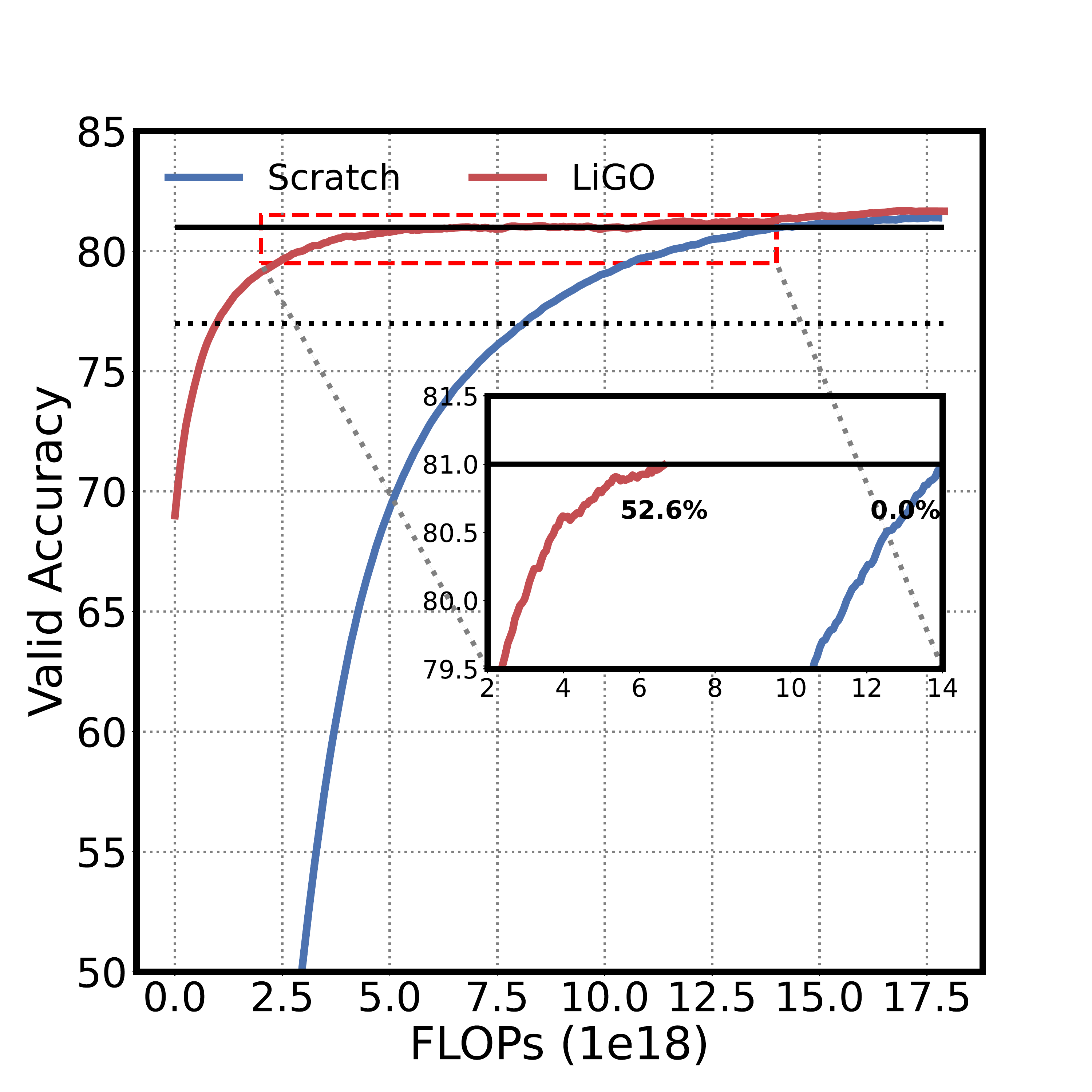} \vspace{-7mm}
         \caption{\scriptsize CaiT-XS$\rightarrow$CaiT-S}
     \end{subfigure}
     \hspace{-6mm}
     \begin{subfigure}[b]{0.5\textwidth}
         \centering
         \includegraphics[width=\textwidth]{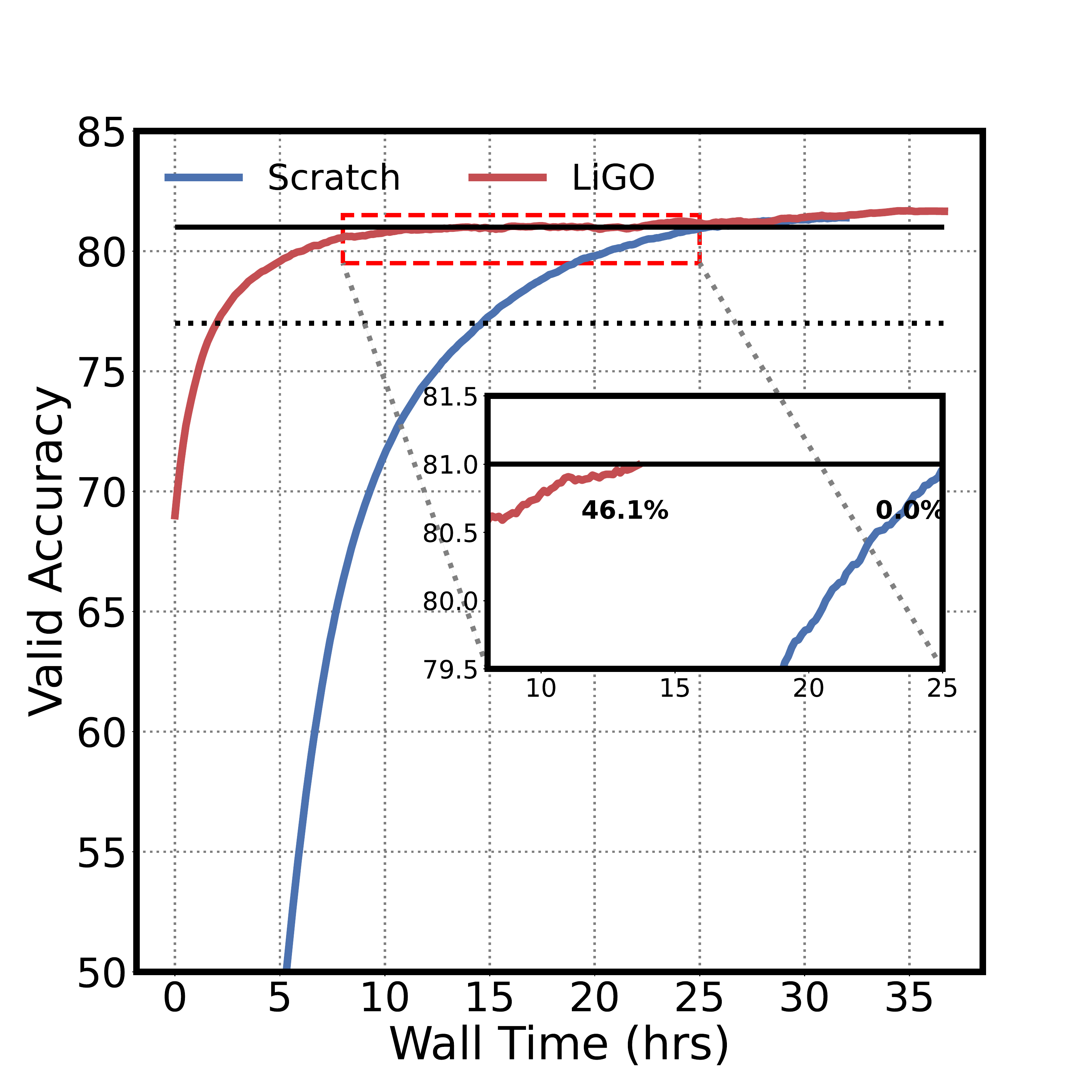} \vspace{-7mm}
         \caption{\scriptsize CaiT-XS$\rightarrow$CaiT-S}
     \end{subfigure} 
     \hspace{-6mm}\vspace{-2mm}
    \caption{\small {Results on CaiT.} (a) Accuracy vs. flops and (b) accuracy vs. wall time for training CaiT-S. \ours saves flops by $52.6\%$ and wall time by $46.1\%$ over training from scratch on ImageNet.}
    \label{fig:cait} 
\end{figure}

In addition to DeiT~\citep{touvron2021training}, we perform additional experiments with CaiT~\citep{touvron2021going} on ImageNet and find that while reusing CaiT-XS, LiGO offers about 52.6\% savings in FLOPs and 46.1\% savings in wall time over the CaiT-S training from scratch (see Figure~\ref{fig:cait}). 

\subsection{Task-specific finetuning with LiGO initialization without further pretraining}
\label{sec:glue_lego}

We perform additional experiments by directly finetuning BERT-Base initialized by LiGO (from BERT-Small) without any further pretraining. We observe in Table~\ref{table:glue_lego_init} that the LiGO-initialized model can benefit downstream tasks compared to BERT-Small trained from scratch (1st row vs 2nd row). 

\begin{table}
\definecolor{Gray}{gray}{0.90}
\newcolumntype{a}{>{\columncolor{Gray}}c}
    \caption{\small GLUE performance of different LiGO models. All of the results are based on BERT-Base models with BERT-Small as the base model for LiGO optimization.} \vspace{-4mm}
    \label{table:glue_lego_init}
    \begin{center}
    \resizebox{\linewidth}{!}{
       \begin{tabular}{l|ccccccc|a}
             \Xhline{3\arrayrulewidth} 
              \multirow{2}{*}{\textbf{Method}} & \textbf{SST-2} & \textbf{MNLI} & \textbf{MRPC} & \textbf{CoLA} & \textbf{QNLI} & \textbf{QQP} & \textbf{STS-B} & \textbf{Average} \\
     & \textbf{(Acc.)} & \textbf{(Acc.)} & \textbf{(Acc.)} & \textbf{(Acc.)} & \textbf{(Acc.)} & \textbf{(Acc.)} & \textbf{(Acc.)} & \textbf{(Acc.)} \\
             \Xhline{2\arrayrulewidth}  
            BERT-Small (Scratch)  & 87.21  & 77.56  & 82.11 & 59.93  & 85.06  & 85.82  & 84.99  & 80.38 \\
            BERT-Base (\ours Init)  & 88.15 & 77.62  & 82.53 & 60.70  & 85.79  &  86.65 & 85.83 & 81.04 \\
            BERT-Base (\ours Init + Pretrain)  & 88.42 & 79.29  & 84.31 & 62.09  & 88.07  &  88.81 & 87.00 & 82.57 \\
            BERT-Base (Scratch)  & 88.19 & 78.43  & 85.78 & 62.09  & 87.06  &  87.18 & 86.99 & 82.25 \\
            \Xhline{3\arrayrulewidth} 
        \end{tabular}
        }
    \end{center}
\end{table}

\subsection{GLUE Performance using AdapterFusion}
\label{sec:adapter}

\begin{table}
\definecolor{Gray}{gray}{0.90}
\newcolumntype{a}{>{\columncolor{Gray}}c}
    \caption{\small Downstream performance using AdapterFusion~\citep{pfeiffer2020adapterfusion} on GLUE Benchmark. All of the results are based on BERT-Base models trained using different baselines.} \vspace{-4mm}
    \label{table:adapter}
    \begin{center}
    \resizebox{\linewidth}{!}{
       \begin{tabular}{l|cc|ccccccc|a}
             \Xhline{3\arrayrulewidth} 
              \multirow{2}{*}{\textbf{Method}} & \textbf{Savings} & \textbf{Savings} & \textbf{SST-2} & \textbf{MNLI} & \textbf{MRPC} & \textbf{CoLA} & \textbf{QNLI} & \textbf{QQP} & \textbf{STS-B} & \textbf{Average} \\
    & \textbf{(FLOPs)} & \textbf{(Walltime)} & \textbf{(Acc.)} & \textbf{(Acc.)} & \textbf{(Acc.)} & \textbf{(Acc.)} & \textbf{(Acc.)} & \textbf{(Acc.)} & \textbf{(Acc.)} & \textbf{(Acc.)} \\
             \Xhline{2\arrayrulewidth}  
             Scratch & -- &  -- & 88.41 & 78.60  & 86.02 & 62.39  & 87.62  & 88.02  & 86.52 & 82.51 \\
            StackBERT & 34.1\%  & 33.3\% & 88.78  & 79.80  & 85.43 & 59.56  & 87.71  & 89.19  & 86.27  & 82.39 \\
            MSLT & 34.9\% & 30.0\%   & 88.41 & 78.35  & 83.15 & 63.97  & 86.19  & 88.20 & 86.42 & 82.10 \\
            KI & -5.7\% & -13.9\%  & 88.94 & 78.84  & 84.00 & 64.61  & 86.75  & 88.19  & 87.93 & 82.75 \\
            bert2BERT & 29.0\%  & 25.1\%  & 88.47 & 80.53  & 85.50  & 62.33 &  88.57 & 86.72 & 87.10 & 82.75\\
            \Xhline{1\arrayrulewidth}
            \ours & 44.7\%  & 40.5\% & 88.45 & 80.01  & 84.67 & 63.05  & 88.06  &  88.92 & 87.00 & 82.88 \\
            \Xhline{3\arrayrulewidth} 
        \end{tabular} \vspace{-6mm}
        }
    \end{center}
\end{table}

\ours is mainly proposed for improving efficiency of the pre-training stage and hence is compatible with various finetuning schemes like full model finetuning, adapters~\citep{houlsby2019parameter,pfeiffer2020adapterfusion} or prompt tuning~\citep{lester2021power,jia2022visual} for adaptation to downstream tasks. We test BERT-Base models trained using different baselines by using adapterfusion~\citep{pfeiffer2020adapterfusion} instead of full finetuning on GLUE benchmark. Table~\ref{table:adapter} shows that \ours also achieves on-par performance with model trained from scratch under adapter-based tuning with 44.7\% savings in FLOPs abd 40.5\% savings in wall time. This shows that \ours does not harm the model generalization capability when adapters are used as a parameter-efficient finetuning strategy for transferring a trained model to downstream datasets.

\subsection{Initial results on billion+ parameter models}
\label{sec:scalability}

Our extensive experiments on BERT~\citep{devlin2018bert}, RoBERTa~\citep{liu2019roberta}, GPT2~\citep{radford2019language}, DeiT~\citep{touvron2021training}  and CaiT~\citep{touvron2021going} show that LiGO can consistently improve transformer training efficiency over the traditional way of training from scratch across domains and model sizes. One interesting future direction of our work is scaling LiGO to very large models with parameters more than 100B, such as GPT3~\citep{brown2020language}. While we currently do not possess the compute resources for this extreme large-scale study, we perform a preliminary experiment on GPT2-1.5B~\citep{radford2019language} by using GPT2-Medium as the initialization. We train for 15k steps on C4 dataset~\citep{raffel2020exploring} and find that our proposed LiGO saves about 39\% computation cost (FLOPs) of training GPT2-1.5B from scratch to reach the same log perplexity (which is 3.3). We believe that it is imperative to study the extent to which the benefits of LiGO remain at the scale on which the modern large language models are trained. We hope to cover this in our future work.

\end{document}